\documentclass{article}
\usepackage{tikz}
\usepackage[linesnumbered, ruled, vlined]{algorithm2e}
\SetKwProg{Fn}{Function}{}{end}\SetKwFunction{FRecurs}{FnRecursive}%

\usepackage[noend]{algpseudocode}
\usepackage{amsmath}
\usepackage{microtype}
\usepackage{flexisym}
\usepackage{float}
\usepackage[hyphens]{url}
\usepackage{subcaption}
\providecommand{\keywords}[1]{\textbf{\textit{Keywords}} #1}
\usepackage{times}
\usepackage{amsthm}
\usepackage[fleqn]{mathtools}
\newcommand\independent{\protect\mathpalette{\protect\independenT}{\perp}}\def\independenT#1#2{\mathrel{\rlap{$#1#2$}\mkern2mu{#1#2}}}
\newcommand{\tb}[1]{\textbf{#1}}
\newcommand{\mf}[1]{\mathbf{#1}}

\newtheorem{lem}{Lemma}

\newtheorem{thm}{Theorem}

\newtheorem{prop}{Proposition}

\newtheorem{defn}{Definition}

\usepackage[left=1in,right=1in,top=1in,bottom=1.5in]{geometry}

\usepackage{mathtools} 
\usepackage[shortlabels]{enumitem}

\usepackage{authblk}
\usepackage{blindtext}
\newcommand{\orcid}[1]{\href{#1}}

\begin{document}
\title{Adjustment Criteria for Recovering Causal Effects from Missing Data}
%

%

\author[1]{Mojdeh Saadati 0000-0002-9092-8464}
\author[2]{Jin Tian 0000-0001-5313-1600}
\affil[1,2]{Department of Computer Science, Iowa State University} 
\affil[1,2]{\{msaadati, jtian\}@iastate.edu}


\maketitle 
\setcounter{footnote}{0}

\begin{abstract}
    Confounding bias, missing data, and selection bias are three common obstacles to valid causal inference in the data sciences. Covariate adjustment is the most pervasive technique for recovering casual effects from confounding bias. In this paper we introduce a covariate adjustment formulation for controlling confounding bias in the presence of missing-not-at-random data and develop a necessary and sufficient condition for recovering causal effects using the adjustment. We also introduce an adjustment formulation for controlling both confounding and selection biases in the presence of missing data and develop a necessary and sufficient condition for valid adjustment. Furthermore, we present an algorithm that lists all valid adjustment sets and an algorithm that finds a valid adjustment set containing the minimum number of variables, which are useful for researchers interested in selecting adjustment sets with desired properties.


\end{abstract}
\section{Introduction}

	Discovering causal relationships from observational data has been an important task in empirical sciences, for example, assessing the effect of a drug on curing diabetes, a fertilizer on growing agricultural products, and an advertisement on the success of a political party. 
	One major challenge to estimating the effect of a treatment on an outcome from observational data is the existence of 
	\emph{confounding bias} - i.e., the lack of control on the effect of spurious variables on the outcome. 
	This issue is formally addressed as the \emph{identifiability problem} in \cite{Pearl:2009:CMR:1642718}, which concerns with computing the effect of a set of treatment variables ($\tb{X}$) on a set of outcome variables ($\tb{Y}$), denoted by $P(\tb{y}\mid do(\tb{x}))$, given observed probability distribution $P(\tb{V})$ and a causal graph $G$, where  $P(\tb{V})$ corresponds to the observational data and $G$ is a directed acyclic graph (DAG) representing qualitative causal relationship assumptions between variables in the domain. The effect $P(\tb{y} \mid do(\tb{x}))$ may not be equal to its probabilistic counterpart $P(\tb{y}\mid\tb{x})$ due to the existence of variables, called \emph{covariates}, that affect both the treatments and outcomes, and the difference is known as confounding bias. For example, Fig.~\ref{fig-1}(a) shows a causal graph where variable $Z$ is a covariate for estimating the effect of $X$ on $Y$. 
	
	Confounding bias problem has been studied extensively in the field. In principle the identifiability problem can be solved using a set of causal inference rules called \emph{do-calculus} \cite{pearl1995causal}, and complete identification algorithms have been developed \cite{tian2002general,Huang:2006:ICB:1597348.1597371,shpitser2006identification}. 
	In practice, however, the most widely used method for controlling the confounding bias is the following ``adjustment formula'' 
	\begin{equation*}
	P(\tb{y} \mid do(\tb{x})) = \sum_{\tb{z}}^{} P(\tb{y} \mid \tb{x}, \tb{Z} = \tb{z})P(\tb{Z} = \tb{z}), 
	\end{equation*}
which dictates that the causal effect $P(\tb{y} \mid do(\tb{x}))$ can be computed by \emph{controlling} for a set of covariates $\tb{Z}$. Pearl provided a back-door criterion under which a set $\tb{Z}$ makes the adjustment formula hold \cite{pearl1995causal}.  

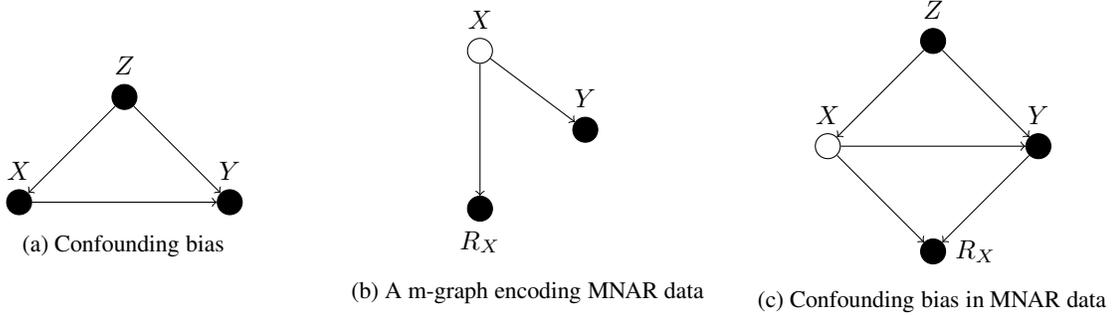
\begin{figure}[H] 
	\begin{center}
		\begin{tabular}{ccc}
	\begin{subfigure}[normal]{0.3\linewidth}
	\begin{tikzpicture}[scale = 0.7]
	\node[circle,draw , fill = black, ,label=above:{$X$}] (X) at (0,3) {};
	\node[circle,draw, fill = black, ,label=above:{$Y$}] (Y) at (4,3) {};
	\node[circle,draw, fill = black, ,label=above:{$Z$}] (Z) at (2,5) {};
	\draw[->] (X) -- (Y);
	\draw[->] (Z) -- (X);
	\draw[->] (Z) -- (Y);
	\end{tikzpicture}
    \centering
	\caption{Confounding bias}
	\end{subfigure}
		&
	\begin{subfigure}[normal]{0.3\linewidth}
	\begin{tikzpicture}[scale = 0.7]
	\node[circle,draw,label=above:{$X$}] (X) at (12,5) {};
	\node[circle,draw,,fill = black, ,label=above:{$Y$}] (Y) at (14,3.5) {};
	\node[circle,draw, ,fill = black, ,label=below:{$R_X$}] (RX) at (12,2) {};
	\draw[->] (X) -- (Y);
	\draw[->] (X) -- (RX);
	\end{tikzpicture}
	\centering
	\caption{A m-graph encoding MNAR data} 
\end{subfigure}
	&
	\begin{subfigure}[normal]{0.3\linewidth}
	\begin{tikzpicture}[scale = 0.7]
	\node[circle,draw,label=above:{$X$}] (X) at (21,3) {};
	\node[circle,draw,fill = black, ,label=above:{$Y$}] (Y) at (25,3) {};
	\node[circle,draw,,fill = black, ,label=above:{$Z$}] (Z) at (23,5) {};
	\node[circle,draw,,fill = black, ,label=right:{$R_X$}] (RX) at (23,1) {};
	\draw[->] (X) -- (Y);
	\draw[->] (Z) -- (X);
	\draw[->] (Z) -- (Y);
	\draw[->] (X) -- (RX);
	\draw[->] (Y) -- (RX);
	
	\end{tikzpicture}
	\centering
		\caption{Confounding bias in MNAR data} 
\end{subfigure}
		\end{tabular}
		\caption{Examples for confounding bias in MNAR data}
		\label{fig-1}
	\end{center}

\end{figure}  
			
	Another major challenge to valid causal inference is the missing data problem, which occurs when some variable values are missing from observed data. Missing data is a common problem in empirical sciences. Indeed there is a large literature on dealing with missing data in diverse disciplines including statistics, economics, social sciences, and machine learning. To analyze data with missing values, it is imperative to understand the mechanisms that lead to missing data. The seminal work by Rubin  
	\cite{rubin1976inference} classifies missing data mechanisms into three categories: \emph{missing completely at random (MCAR)}, \emph{missing at random (MAR)}, and \emph{missing not at random (MNAR)}. Roughly speaking, the mechanism is MCAR if whether variable values are missing is completely independent of the values of variables in the data set; the mechanism is MAR when missingness is independent of the missing values given the observed values; and the mechanism is MNAR if it is neither MCAR nor MAR. For example, assume that in a study of the effect of family income ($FI$) and parent's education level ($PE$) on the quality of child's education ($CE$), some respondents chose not to reveal their child's education quality for various reasons. Fig.~\ref{fig-2} shows causal graphs representing the three missing data mechanisms where $R_{CE}$ is an indicator variable such that $R_{CE}=0$ if the $CE$ value is missing and $R_{CE}=1$ otherwise. In these graphs solid circles represent always-observed variables and hollow circles represent variables that could have missing values. The model in Fig.~\ref{fig-2}(a) is MCAR, e.g., respondents decide to reveal the child's education quality based on coin-flips. The model in Fig.~\ref{fig-2}(b) is MAR, where respondents with higher family income have a higher chance of revealing the child's education quality; however whether the $CE$ values are missing is independent of the actual values of $CE$ given the $FI$ value. The model in Fig.~\ref{fig-2}(c) is MNAR, where respondents with higher child's education quality have a higher chance of revealing it, i.e., whether the $CE$ values are missing depends on the actual values of $CE$.

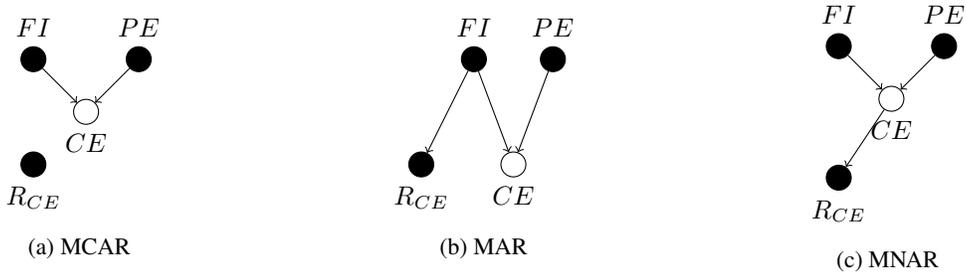
\begin{figure} [H]
	\begin{center}
		\begin{tabular}{ccc}
	\begin{subfigure}[normal]{0.3\linewidth}\
		\centering
		\begin{tikzpicture}[scale = 0.7]
		\node[circle,draw,,fill = black, ,label=above:{$FI$}] (FI) at (0,3) {};
		\node[circle,draw,,label=below:{$CE$}] (CE) at (1,2) {};
		\node[circle,draw,,fill = black, ,label=above:{$PE$}] (PE) at (2,3) {};
		\node[circle,draw,fill = black, ,label=below:{$R_{CE}$}] (R) at (0,1) {};
		\draw[->] (FI) -- (CE);
		\draw[->] (PE) -- (CE);
	\end{tikzpicture}
	\caption{MCAR}  
	\end{subfigure}
		&
	\begin{subfigure}[normal]{0.3\linewidth}
		\begin{tikzpicture}[scale = 0.7]
		\centering
		\node[circle,draw,,fill = black, ,label=above:{$FI$}] (FI) at (1,3) {};
		\node[circle,draw,,label=below:{$CE$}] (CE) at (1.75,1) {};
		\node[circle,draw,,fill = black, ,label=above:{$PE$}] (PE) at (2.5,3) {};
		\node[circle,draw,,fill = black, ,label=below:{$R_{CE}$}] (R) at (0,1) {};
		\draw[->] (FI) -- (CE);
		\draw[->] (PE) -- (CE);
		\draw[->] (FI) -- (R);	
		\end{tikzpicture}
	\centering	
	\caption{MAR}   
\end{subfigure}
	&
	\begin{subfigure}[normal]{0.3\linewidth}
		\begin{tikzpicture}[scale = 0.7]
		\centering
		\node[circle,draw,,fill = black, ,label=above:{$FI$}] (FI) at (3,3) {};
		\node[circle,draw,label=below:{$CE$}] (CE) at (4,2) {};
		\node[circle,draw,,fill = black, ,label=above:{$PE$}] (PE) at (5,3) {};
		\node[circle,draw,,fill = black, ,label=below:{$R_{CE}$}] (R) at (3,0.5) {};
		\draw[->] (FI) -- (CE);
		\draw[->] (PE) -- (CE);
		\draw[->] (CE) -- (R);	
		\end{tikzpicture}
		\centering
		\caption{MNAR}   
\end{subfigure}
		\end{tabular}
	    \centering
		\caption{Three types of missing data mechanisms}
		\label{fig-2}
	\end{center}
\end{figure}

It is known that when the data is MAR, the underlying distribution is estimable from observed data with missing values. Then a causal effect is estimable if it is identifiable from the observed distribution \cite{mohan2014graphical}. However, if the data is MNAR, whether a probabilistic distribution or a causal effect is estimable from missing data depends closely on both the query and the exact missing data mechanisms. For example, in the MNAR model in Fig.~\ref{fig-1}(b), $P(X)$ cannot be estimated consistently even if infinite amount of data are collected, while $P(y|do(x))=P(y|x)=P(y|x, R_X=1)$ is estimable from missing data. On the other hand, in the MNAR model in Fig.~\ref{fig-1}(c), $P(y|do(x))$ is not estimable. In the MNAR model in Fig.~\ref{fig-2}(c), neither $P(CE)$ nor $P(CE \mid do(FI))$ can be estimated from observed data with missing values.

	Various techniques have been developed to deal with missing data in statistical inference, e.g., listwise deletion \cite{Little:1986:SAM:21412}, which requires data to be MCAR to obtain unbiased estimates, and multiple imputation \cite{rubin1978multiple}, which requires MAR. Most of the work in machine learning makes MAR assumption and use maximum likelihood based methods (e.g. EM algorithms) \cite{koller2009probabilistic} , with a few work explicitly incorporates missing data mechanism into the model 
\cite{koller2009probabilistic,DBLP:conf/ijcai/MarlinZRS11,DBLP:conf/uai/MarlinZRS07}.
	
The use of graphical models called \emph{m-graphs} for inference with missing data was more recent \cite{mohan2013graphical}. M-graphs provide a general framework for inference with arbitrary types of missing data mechanisms including MNAR. Sufficient conditions for determining whether probabilistic queries (e.g., $P(\tb{y} \mid \tb{x})$ or $P(\tb{x},\tb{y})$) are estimable from missing data are provided in \cite{mohan2013graphical,mohan2014graphical}. General algorithms for identifying the joint distribution have been developed in \cite{Shpitser:2015:MDC:3020847.3020930,tian2017recovering}.

The problem of identifying causal effects $P(\tb{y} \mid do(\tb{x}))$ from missing data in the causal graphical model settings has not been well studied. To the best of our knowledge the only results are the sufficient conditions given in \cite{mohan2014graphical}. The goal of this paper is to provide general conditions under which the causal effects can be identified from missing data using the covariate adjustment formula, which is the most pervasive method in practice for causal effect estimation under confounding bias.   

We will also extend our results to cope with another common obstacles to valid causal inference - \emph{selection bias}. 
Selection bias may happen due to preferential exclusion of part of the population from sampling. To illustrate, consider a study of the effect of diet on blood sugar. If individuals that are healthy and consume less sugar than average population are less likely to participate in the study, then the data gathered is not a faithful representation of the population and biased results will be produced. This bias cannot be removed by sampling more examples or controlling for confounding bias. 
Note that, in some sense, selection bias could be considered as a very special case of missing data mechanisms, where values of all of the variables are either all observed or all missing simultaneously.  Missing data problem allows much richer missingness patterns such that in any particular observation, some of the variables could be observed and others could be missing. Missing data is modeled by introducing individual missingness indicators for each variable (such that $R_X=0$ if $X$ value is missing), while selection bias is typically modeled by introducing a single selection indicator variable ($S$) representing whether a unit is included in the sample or not (that is, if $S=0$ then values of all variables are missing). 

Identifying causal effects from selection bias has been studied in the literature  \cite{bareinboim2014recovering,DBLP:conf/aaai/BareinboimT15}. 
Adjustment formulas for recovering causal effects under selection bias have been introduced and complete graphical criteria have been developed \cite{correa2017causal,correa2018generalized}. However these results are not applicable to the missing data problems which have much richer missingness patterns than could be modeled by selection bias.  
To the best of our knowledge, using adjustment for causal effect identification when the observed data suffers from missing values or both selection bias and missing values has not been studied in the causal graphical model settings. In this paper we will provide a characterization for these tasks.


Specifically, the contributions of this paper are:
	\begin{itemize}
	
		\item We introduce a covariate adjustment formulation for recovering causal effects from missing data, and provide a necessary and sufficient graphical condition for when a set of covariates are valid for adjustment. 
		\item We introduce a covariate adjustment formulation for causal effects identification when the observed data suffer from both selection bias and missing values, and provide a necessary and sufficient graphical condition for the validity of a set of covariates for adjustment.
		\item We develop an algorithm that lists \emph{all} valid adjustment sets in polynomial delay time, and an algorithm that finds a valid adjustment set containing the minimum number of variables. The algorithms are useful for scientists to select adjustment sets with desired properties (e.g. low measurement cost). 
		\end{itemize}
\section{Definitions and Related Work}Each variable will be represented with a capital letter ($X$) and its realized value with the small letter ($x$). We will use bold letters ($\mf{X}$) to denote sets of variables.\\

\noindent{\bf{Structural Causal Models.}} The systematic analysis of confounding bias, missing data mechanisms, and selection bias requires a formal language where the characterization of the underlying data-generating model can be encoded explicitly. We use the language of Structural Causal Models (SCM) \cite{Pearl:2009:CMR:1642718}. 
 In SCMs, performing an action/intervention of setting $\tb{X} {=} \tb{x}$ is represented through the do-operator, $do({\tb{X} {=} \tb{x}})$, which 
induces an experimental distribution $P(\tb{y}|do(\tb{x}))$, known as the causal effect of $\tb{X}$ on $\tb{Y}$. We will  use do-calculus to derive causal expressions from other causal quantities. For a detailed discussion of SCMs and do-calculus, we refer readers to \cite{Pearl:2009:CMR:1642718}. 

Each SCM $M$ has a causal graph $G$ associated to it, with directed arrows encoding direct causal relationships and dashed-bidirected arrows encoding the existence of an unobserved common causes (e.g., see Fig.~\ref{Fig1}). We use typical graph-theoretic terminology $Pa(\mf{C}),Ch(\mf{C}), De(\mf{C}), An(\mf{C})$ representing the union of $\mf{C}$ and respectively the parents, children, descendants, and ancestors of $\mf{C}$. 
 We use $G_{\overline{\mf{C}_1}\underline{\mf{C}_2}}$ to denote the graph resulting from deleting all incoming edges to $\mf{C}_1$ and all outgoing edges from $\mf{C}_2$ in $G$. The expression $(\mf{X} \independent \mf{Y} \mid \mf{Z})_G$ denotes that $\mf{X}$ is d-separated from $\mf{Y}$ given $\mf{Z}$ in the corresponding causal graph $G$ \cite{Pearl:2009:CMR:1642718}(subscript $G$ may be omitted).

\noindent{\bf{Missing Data and M-graphs.}}
To deal with missing data, we use \emph{m-graphs} introduced in \cite{mohan2013graphical} to represent both the data generation model and the missing data mechanisms. M-graphs enhance the causal graph $G$ by introducing a set $\tb{R}$ of binary missingness indicator variables. We will also partition the set of observable variables $\tb{V}$ into $\tb{V}_o$ and $\tb{V}_m$ such that $\tb{V}_o$ is the set of variables that will be observed in all data cases and $\tb{V}_m$ is the set of variables that are missing in some data cases and observed in other cases. 
Every variable $V_i\in \tb{V}_m$ is associated with a variable $R_{V_i}\in \tb{R}$ such that, in any observed data case, $R_{V_i}=0$ if the value of corresponding $V_i$ is missing and $R_{V_i}=1$ if $V_i$ is observed. We assume that $\tb{R}$ variables may not be parents of variables in $\tb{V}$, since $\tb{R}$ variables are missingness indicator variables and we assume that the data generation process over $\tb{V}$ variables does not depend on the missingness mechanisms. For any set $\tb{C}\subseteq \tb{V}_m$, let $\tb{R}_{\tb{C}}$ represent the set of $\tb{R}$ variables corresponding to variables in $\tb{C}$.
See Fig.~\ref{fig-2} for examples of m-graphs, in which we use solid circles to represent always observed variables in $\tb{V}_o$ and $\tb{R}$, and hollow circles to represent  partially observed variables in $\tb{V}_m$. \\

	
\noindent{\bf{Causal Effect Identification by Adjustment.}}
	Covariate adjustment is the most widely used technique for identifying causal effects from observational data. Formally, 
\begin{defn}[Adjustment Formula \cite{Pearl:2009:CMR:1642718}]\label{def-adj}
	Given a causal graph $G$ over a set of variables $\tb{V}$, a set $\tb{Z}$ is called \emph{covariate adjustment} (or adjustment for short) for estimating the causal effect of $\tb{X}$ on $\tb{Y}$, if,  for any distribution $P(\tb{V})$ compatible with $G$, it holds that 
	\begin{equation}\label{eq-adj}
	P(\tb{y} \mid do(\tb{x})) = \sum_{\tb{z}} P(\tb{y} \mid \tb{x},\tb{z})P(\tb{z}).
	\end{equation}
\end{defn}
Pearl developed the celebrated ``Backdoor Criterion'' to determine whether a set is admissible for adjustment \cite{pearl1995causal} given in the following:   
\begin{defn}[Backdoor Criterion] \label{defBackDoor}
	A set of variables $\tb{Z}$ satisfies the backdoor criterion relative to a pair of variables $(\tb{X}, \tb{Y})$ in a causal graph $G$ if:
	\begin{enumerate}[a)]
		\item 
		No node in $\tb{Z}$ is a descendant of $\tb{X}$, and 
		\item
		$\tb{Z}$ blocks every path between $\tb{X}$ and $\tb{Y}$ that contains an arrow into $\tb{X}$.
	\end{enumerate} 	
\end{defn}
Complete graphical conditions have been derived for determining whether a set is admissible for adjustment 
\cite{Shpitser:2010:VCA:3023549.3023612,vanderZander:2014:CSA:3020751.3020845,perkovic2017complete} as follows. 
	\begin{defn}[Proper Causal Path] \label{defProper}
	A proper causal path from a node $X \in \tb{X}$ to a node $Y \in \tb{Y}$ is a causal path (i.e., a directed path) which does not intersect $\tb{X}$ except at the beginning of the path.
\end{defn} 
	
\begin{defn}[Adjustment Criterion \cite{Shpitser:2010:VCA:3023549.3023612}]\label{def-adj-cri}
	A set of variables $\tb{Z}$ satisfies the adjustment criterion relative to a pair of variables $(\tb{X}, \tb{Y})$ in a causal graph $G$ if:
\begin{enumerate}[a)]
	\item 
	No element of \tb{Z} is a descendant in $G_{\overline{X}}$ of any $W \notin \tb{X} $ which lies on a proper causal path from $\tb{X}$ to $\tb{Y}$. 
	\item
	All non-causal paths between $\tb{X}$ and \tb{Y} in $G$ are blocked by \tb{Z}. 
\end{enumerate} 
\end{defn}
A set $\tb{Z}$ is an admissible adjustment for estimating the causal effect of $\tb{X}$ on $\tb{Y}$ by the adjustment formula if and only if it satisfies the adjustment criterion.




	\section{Adjustment for Recovering Causal Effects from Missing Data \label{sec-missing}}
In this section we address the task of recovering a causal effect 	$P(\tb{y} \mid do(\tb{x}))$ from missing data given a m-graph $G$ over observed variables $\tb{V}=\tb{V}_o\cup \tb{V}_m$ and missingness indicators $\tb{R}$. The main difference with the well studied identifiability problem \cite{Pearl:2009:CMR:1642718}, where we attempt to identify $P(\tb{y} \mid do(\tb{x}))$ from the joint distribution $P(\tb{V})$, lies in that, given data corrupted by missing values, $P(\tb{V})$ itself may not be recoverable. Instead, a distribution like $P(\tb{V}_o, \tb{V}_m, \tb{R}=1)$ is assumed to be estimable from observed data cases in which all variables in $\tb{V}$ are observed (i.e., complete data cases). In general, in the context of missing data, the probability distributions in the form of $P(\tb{V}_o, \tb{W}, \tb{R}_{\tb{W}}=1)$ for any $\tb{W}\subseteq \tb{V}_m$, called \emph{manifest distributions}, are assumed to be estimable from observed data cases in which all variables in $\tb{W}$ are observed (values of variables in $\tb{V}_m\setminus \tb{W}$ are possibly missing).
The problem of recovering probabilistic queries from the manifest distributions has been studied in \cite{mohan2013graphical,mohan2014graphical,Shpitser:2015:MDC:3020847.3020930,tian2017recovering}.

We will extend the adjustment formula for identifying causal effects to the context of missing data based on the following observation which is stated in Theorem 1 in \cite{mohan2013graphical}:
\begin{lem}\label{lm-rec} 
For any $\tb{W}_o, \tb{Z}_o \in \tb{V}_o$ and $\tb{W}_m, \tb{Z}_m \in \tb{V}_m$, $P(\tb{W}_o, \tb{W}_m \mid \tb{Z}_o, \tb{Z}_m, \tb{R}_{\tb{W}_m\cup \tb{Z}_m} =1)$ is recoverable.
\end{lem}

Formally, we introduce the adjustment formula for recovering causal effects from missing data by extending Eq.~(\ref{eq-adj}) as follows.
\begin{defn}[M-Adjustment Formula] \label{defMadjustment} Given a m-graph $G$ over observed variables $\tb{V}=\tb{V}_o\cup \tb{V}_m$ and missingness indicators $\tb{R}$, a set $\tb{Z}\subseteq \tb{V}$ is called a m-adjustment (adjustment under missing data) set for estimating the causal effect of \tb{X} on \tb{Y}, if, for every model compatible with $G$, it holds that 
		\begin{align}\label{eq-m-adj}
		P(\tb{y} \mid do(\tb{x})) = \sum_{\tb{z}} P(\tb{y} \mid \tb{x}, \tb{z} ,\tb{R}_{\tb{W}} = 1)P(\tb{z} \mid \tb{R}_{\tb{W}} = 1),
		\end{align}
		where $\tb{W} = \tb{\tb{V}}_m \cap (\tb{X}\cup\tb{Y}\cup\tb{Z})$.
	\end{defn}
In the above formulation, we allow that the treatments $\tb{X}$, outcomes $\tb{Y}$, and covariates $\tb{Z}$ all could contain $\tb{V}_m$ variables that have missing values. Both terms on the right-hand-side of Eq.~(\ref{eq-m-adj}) are recoverable based on Lemma~\ref{lm-rec}. Therefore the causal effect 	$P(\tb{y} \mid do(\tb{x}))$ is recoverable if it can be expressed in the form of m-adjustment.

We look for conditions under which a set $\tb{Z}$ is admissible as m-adjustment. In principle this can be derived using do-calculus. As an example, consider the m-graph in   
 Fig.~\ref{Fig1} where $R_1, R_2, R_3$ are missingness indicators for $Z_{m1},Z_{m2},Z_{m3}$ respectively. 
We show that $\{Z_{m1}\}$ is m-adjustment admissible for recovering $P(y \mid do(x_1, x_2))$ using do-calculus derivation as follows:   
	\begin{align}
	&P(y \mid do(x_1, x_2))\nonumber
	\\&= P(y \mid do(x_1, x_2), R_1=1)
	\\&= \sum_{Z_{m1}}P(y \mid do(x_1, x_2),Z_{m1}, R_1 = 1)P(Z_{m1} \mid R_1 =1, do(x_1, x_2) )
	\\&= \sum_{Z_{m1}}P(y \mid do(x_1, x_2),Z_{m1}, R_1 =1)P(Z_{m1} \mid R_1 =1)
	\\& =\sum_{Z_{m1}}P(y \mid x_1, x_2 ,Z_{m1}, R_1 =1)P(Z_{m1} \mid R_1 =1).
	\end{align}

	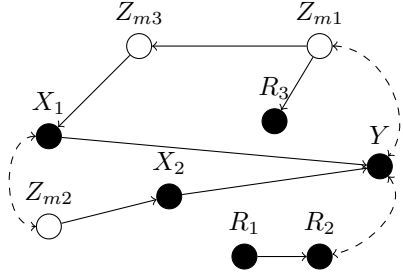
\begin{figure}
		\centering
		\begin{tikzpicture}[scale = 0.40]
		\node[circle,draw,fill = black, ,label=above:{$X_1$}] (X1) at (0,7) {};
		\node[circle,draw,fill = black, ,label=above:{$X_2$}] (X2) at (4,5) {};
		\node[circle,draw,,label=above:{$Z_{m1}$}] (Z1) at (9,10) {};
		\node[circle,draw,,label=above:{$Z_{m3}$}] (Z3) at (3,10) {};
		\node[circle,draw,fill = black,label=above:{$R_3$}] (R3) at (7.5,7.5) {};
		\node[circle,draw,fill = black, ,label=above:{$Y$}] (Y) at (11,6) {};
		\node[circle,draw,,label=above:{$Z_{m2}$}] (Z2) at (0,4) {};
		\node[circle,draw,fill = black,label=above:{$R_1$}] (R1) at (6.5,3) {};
		\node[circle,draw,fill = black,label=above:{$R_2$}] (R2) at (9,3) {};

		\draw[dashed, <->] (X1)to [out=-180,in= -180] (Z2);
		\draw[->] (Z3)-- (X1);
		\draw[->] (X1)-- (Y);   
		\draw[->] (X2)-- (Y); 
		\draw[ ->] (Z2)-- (X2); 	
		\draw[dashed, <->] (Z1)to [out=0,in=50] (Y);
		\draw[dashed,<->] (Y) to [out=-50,in=10] (R2);
		\draw[->] (Z1)-- (R3);
		\draw[->] (Z1)-- (Z3);
		\draw[ ->] (R1)-- (R2);

		\end{tikzpicture}
		\centering
		\caption[loftitle]{An example of m-adjustment in a MNAR model}  
		\label{Fig1}              
	\end{figure}
	
In general using do-calculus to recover causal effects is difficult due to many possible ways of applying do-calculus rules in every stage of the derivation. Intuitively, we can start with the adjustment formula~(\ref{eq-adj}), consider an adjustment set as a candidate m-adjustment set, and then check for needed conditional independence relations. Based on this intuition, we obtain  		
	a straightforward sufficient condition for a set \tb{Z} to be a m-adjustment set as follows. 
	\begin{prop}
	\label{prop-adj}
	A set $\tb{Z}$ is a m-adjustment set for estimating the causal effect of \tb{X} on \tb{Y} if, letting $\tb{W} = \tb{\tb{V}}_m \cap (\tb{X}\cup\tb{Y}\cup\tb{Z})$, 
	\begin{enumerate}[a)]
	\item 
	\tb{Z} satisfies the adjustment criterion (Def.~\ref{def-adj-cri}),   	
	\item 
	$\tb{R}_{\tb{W}}$ is d-separated from \tb{Y} given \tb{X}, \tb{Z}, i.e., $(\tb{Y} \independent \tb{R}_{\tb{W}} \mid \tb{X, Z})$, 
	\item 
	 $\tb{Z}$ is d-separated from $\tb{R}_{\tb{W}}$, i.e., $( \tb{Z} \independent {\tb{R}_{\tb{W}}})$.  
\end{enumerate}
\end{prop}
\begin{proof}\let\qed\relax
Condition (a) makes sure that the causal effect can be identified in terms of the adjustment formula ~(\ref{eq-adj}). Then given Conditions (b) and (c),  Eq.~(\ref{eq-adj}) is equal to Eq.~(\ref{eq-m-adj}). 
\end{proof}
For example, in Fig.~\ref{Fig1}, $\{Z_{m1}\}$, $\{Z_{m3}\}$, and $\{Z_{m1},Z_{m3}\}$ all satisfy the back-door criterion (and therefore the adjustment criterion), however only $\{Z_{m1}\}$ satisfies the conditions in Proposition~\ref{prop-adj} ($\{Z_{m3}\}$, and  $\{Z_{m1},Z_{m3}\}$ do not satisfy Condition (c) because $Z_{m3}$ is not d-separated from $R_3$).

However this straightforward criterion 
in Proposition~\ref{prop-adj} is not necessary. To witness, consider the set $\{V_{m1},V_{m2}\}$ in Fig.~\ref{Fig3} which satisfies the back-door criterion but not the conditions in Proposition~\ref{prop-adj} because $V_{m2}$ is not d-separated from $R_2$. Still, it can be shown that $\{V_{m1},V_{m2}\}$ is a m-adjustment set by do-calculus derivation as follows:
\begin{align}
&P(y \mid do(x)) \nonumber\\
&= P(y \mid do(x), R_1=1, R_2 = 1) 
\\&= \sum_{v_{m1}, v_{m2}}P(y \mid do(x),v_{m1},v_{m2}, R_1 =1, R_2 = 1)P(v_{m1},v_{m2} \mid R_1 =1, R_2 = 1, do(x) )
\\&= \sum_{v_{m1}, v_{m2}} P(y \mid do(x),v_{m1},v_{m2}, R_1 =1,R_2 =1)P(v_{m1} ,v_{m2}\mid R_1 =1, R_2 = 1)
\\&= \sum_{v_{m1}, v_{m2}} P(y \mid x,v_{m1},v_{m2}, R_1 =1,R_2 =1)P(v_{m1} ,v_{m2}\mid R_1 =1, R_2 = 1).\label{eq-eg-madj}
\end{align}

\begin{figure}
	\centering
	\begin{tikzpicture}
	\node[circle,draw,fill = black,label=above:{$R_1$}] (R1) at (1,3) {};
	\node[circle,draw,label=above:{$V_{m2}$}] (V2) at (2,5) {};
	\node[circle,draw,fill = black,label=above:{$R_2$}] (R2) at (3,3) {};
	\node[circle,draw,fill = black,label=above:{X}] (X) at (5,3) {};
	\node[circle,draw,label=above:{$V_{m1}$}] (V1) at (7,6) {};
	\node[circle,draw,fill = black,label=above:{$Y$}] (Y) at (9,3) {};
	
	\draw[->] (V2)-- (R1);
	\draw[->] (V2)-- (R2);   
	\draw[->] (V2)-- (X); 
	\draw[ ->] (V1)-- (X); 	
	\draw[->] (V1)-- (Y);
	\draw[->] (X)-- (Y);
	\end{tikzpicture}
	\centering
	\caption[loftitle]{In this m-graph $V_{m2}$ is not d-separated from $R_2$. However, \{$V_{m2}$, $V_{m1}$\} is an admissible m-adjustment set.}             
	\label{Fig3}              
\end{figure}
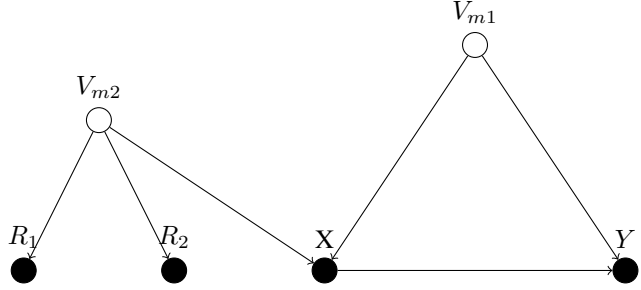

Next we introduce a complete criterion to determine whether a covariate set is admissible as m-adjustment to recover causal effects from missing data, extending the existing work on adjustment \cite{Shpitser:2010:VCA:3023549.3023612,vanderZander:2014:CSA:3020751.3020845,correa2017causal,correa2018generalized,perkovic2017complete}.	
	\begin{defn}[M-Adjustment Criterion]\label{defMcriterion} Given a m-graph $G$ over observed variables $\tb{V}=\tb{V}_o\cup \tb{V}_m$ and missingness indicators $\tb{R}$, and disjoint sets of variables $\tb{X, Y, Z}\subseteq \tb{V}$, letting $\tb{W} = \tb{V}_m \cap (\tb{X}\cup\tb{Y}\cup\tb{Z})$, $\tb{Z}$ satisfies the m-adjustment criterion relative to the pair ($\tb{X}, \tb{Y}$) if  
		\begin{enumerate}[a)]
			\item 
			No element of \tb{Z} is a descendant in $G_{\overline{X}}$ of any $W \notin \tb{X} $ which lies on a proper causal path from \tb{X} to \tb{Y}. 
			\item
			All non-causal paths between \tb{X} and \tb{Y} in $G$ are blocked by \tb{Z} and $\tb{R}_{\tb{W}}$. 	
			\item 
			$\tb{R}_{\tb{W}}$ is d-separated from \tb{Y} given \tb{X} under the intervention of do(\tb{x}), i.e., $(\tb{Y} \independent \tb{R}_{\tb{W}} \mid \tb{X})_{G_{\overline{X}}}$.	
			\item 
			Every $X \in \tb{X} $ is either a non-ancestor of $\tb{R}_{\tb{W}}$ or it is d-separated from \tb{Y} in $G_{\underline{X}}$, i.e., $\forall X \in \tb{X} \cap An(\tb{R}_{\tb{W}}), (X \independent \tb{Y})_{G_{\underline{X}}}$.  
		\end{enumerate}
	\end{defn}
\begin{thm}[M-Adjustment] \label{thm-Madj-com} A set \tb{Z} is a m-adjustment set for recovering causal effect of \tb{X} on \tb{Y} by the m-adjustment formula in Def.~\ref{defMadjustment} if and only if it satisfies the m-adjustment criterion in Def.~\ref{defMcriterion}.  
\end{thm}
The proof of Theorem~\ref{thm-Madj-com} is presented in the Appendix B.
	Conditions (a) and (b) in Def.~\ref{defMcriterion} echo the adjustment criterion in Def.~\ref{def-adj-cri} and it can be shown that if $\tb{Z}$ satisfies the m-adjustment criterion then it satisfies the adjustment criterion (using the fact that no variables in $\tb{R}$ can be parents of variables in $\tb{V}$). In other words, we only need to look for m-adjustment sets from admissible adjustment sets.
	
	As an example consider Fig. \ref{Fig3}. Both $\{V_{m1}\}$ and $\{V_{m1},V_{m2}\}$ satisfy the m-adjustment criterion (and the adjustment criterion too). According to Theorem~\ref{thm-Madj-com}, $P(y \mid do(x))$ can be recovered from missing data by m-adjustment given in Eq.~(\ref{eq-eg-madj}), and can also by recovered as
	\begin{align}
	  P(y\mid do(x)) = \sum_{v_{m1}} p(y \mid x,v_{m1}, R_1 =1)P(v_{m1}\mid R_1 =1).
	\end{align}
	
\subsection{Estimating m-adjustment}
Covariate adjustment is arguably the most widely used method for causal effect estimation in practice. 
A naive approach to estimating $P(\tb{y} \mid do(\tb{x}))$ is directly using Eq.~(\ref{eq-adj}) and estimating the conditional probability distribution of $\tb{Y}$ given $\tb{X}=\tb{x}$ for each possible value of $\tb{Z}$. However, this approach faces computational and sample complexity challenges when $\tb{Z}$ is high dimensional. The number of different values of $\tb{Z}$ grows exponentially in the cardinality of $\tb{Z}$, and the number of samples falling under each $\tb{Z}$ value may be too small to provide a reliable estimate of the conditional distribution.

Robust weighting-based statistical estimation procedures have been developed for estimating the adjustment formula, such as the inverse-probability or stabilized weighting (IPW, SW) \cite{robins2000marginal}, 
 to circumvent these issues with great practical success. These procedures are based on the following rewriting of the adjustment formula  
\begin{equation}
P(\tb{y} \mid do(\tb{x})) = \sum_{\tb{z}} \frac{P(\tb{y},\tb{x},\tb{z})}{P(\tb{x} \mid \tb{z})}.    
\end{equation}
If a reliable estimate of the conditional distribution $P(\tb{x} \mid \tb{z})$ could be obtained, known as the ``propensity score'' \cite{Pearl2016CausalPrimer}, then the causal effect could be estimated by ``weighting'' every observed sample by the factor $1/P(\tb{x} \mid \tb{z})$, leading to the widely used ``inverse probability weighed (IPW) estimator'' \cite{robins2000marginal}.
 In practice, $P(\tb{x} \mid \tb{z})$ is estimated from data by assuming some parametric model (often a logistic regression model).

Next, we show that IPW style estimator could be constructed in the presence of missing data if the causal effect can be estimated using the m-adjustment formula. We rewrite the m-adjustment formula ~(\ref{eq-m-adj}) as follows:
\begin{align}\label{eq-IPW}
&P(\tb{y} \mid do(\tb{x})) \nonumber 
\\&= \sum_{\tb{z}} P(\tb{y} \mid \tb{x}, \tb{z} ,\tb{R}_{\tb{W}} = 1)P(\tb{z} \mid \tb{R}_{\tb{W}} = 1)
\\&= \sum_{\tb{z}} \frac{P(\tb{y} , \tb{x}, \tb{z} \mid \tb{R}_{\tb{W}} = 1)}{P(\tb{x}, \tb{z} \mid \tb{R}_{\tb{W}} = 1))}P(\tb{z} \mid \tb{R}_{\tb{W}} = 1)
\\&= \sum_{\tb{z}} \frac{P(\tb{y} , \tb{x}, \tb{z} \mid \tb{R}_{\tb{W}} = 1)}{P(\tb{x} \mid \tb{z}, \tb{R}_{\tb{W}} = 1)P(\tb{z} \mid \tb{R}_{\tb{W}} = 1)}P(\tb{z} \mid \tb{R}_{\tb{W}} = 1)
\\&= \sum_{\tb{z}} \frac{P(\tb{y} , \tb{x}, \tb{z} \mid \tb{R}_{\tb{W}} = 1)}{P(\tb{x} \mid \tb{z}, \tb{R}_{\tb{W}} = 1)} \label{eq-mipw}
\end{align}

Based on Eq. (\ref{eq-mipw}), all the weighting-based techniques developed for estimating the adjustment formula could be directly extended to estimate the m-adjustment formula by using data cases where all the variables in $\tb{X}\cup\tb{Y}\cup\tb{Z}$ are observed. 

In fact, a popular method to deal with missing data in practice is to use IPW techniques while only using data cases for which all the relevant variables are observed assuming the covariates are admissible for adjustment. However, this practice may lead to biased estimation if the covariates are not admissible for m-adjustment. As implied by the necessary and sufficient result stated in Theorem~\ref{thm-Madj-com}, this method is justified only if the covariates are admissible for m-adjustment which is often a stronger requirement than admissible for adjustment.

	\section{Listing M-Adjustment Sets \label{sec-list}}
	In the previous section we provided a criterion under which a set of variables \tb{Z} is an admissible m-adjustment set for recovering a causal effect. It is natural to ask how to find an admissible set. In reality, it is common that more than one set of variables are admissible. In such situations it is possible that some m-adjustment sets might be preferable over others based on various aspects such as feasibility, difficulty, and cost of collecting variables. Next we first present an algorithm that systematically lists all m-adjustment sets and then present an algorithm that finds a minimum m-adjustment set. These algorithms provide flexibility for researchers to choose their preferred adjustment set based on their needs and assumptions.    
  
	\subsection{Listing all admissible sets}
	It turns out in general there may exist exponential number of m-adjustment sets. To illustrate, we look for possible m-adjustment sets in the m-graph in 
      Fig.~\ref{Fig5} for recovering the causal effect $p(\tb{y}\mid do(\tb{x}))$ (this graph is adapted from a graph in \cite{correa2018generalized}). A valid m-adjustment set $\tb{Z}$ needs to close all the $k$ non-causal paths from $X$ to $Y$. $\tb{Z}$ must contain at least one variable in $\{V_{i1}, V_{i2}, V_{i3}\}$ for each $i=1, \ldots, k$. Therefore, to close each path, there are 7 possible $\tb{Z}$ sets, and for $k$ paths, we have total $7^k$ $\tb{Z}$ sets as potential m-adjustment sets. For each of them, Conditions (c) and (d) in Def.~\ref{defMcriterion} are satisfied because $ (\tb{R} \independent Y \mid X)_{G_{\overline{X}}}$ and $X$ is not an ancestor of any $\tb{R}$ variables. We obtain that there are at least $7^k$ number of m-adjustment sets. \\ 
	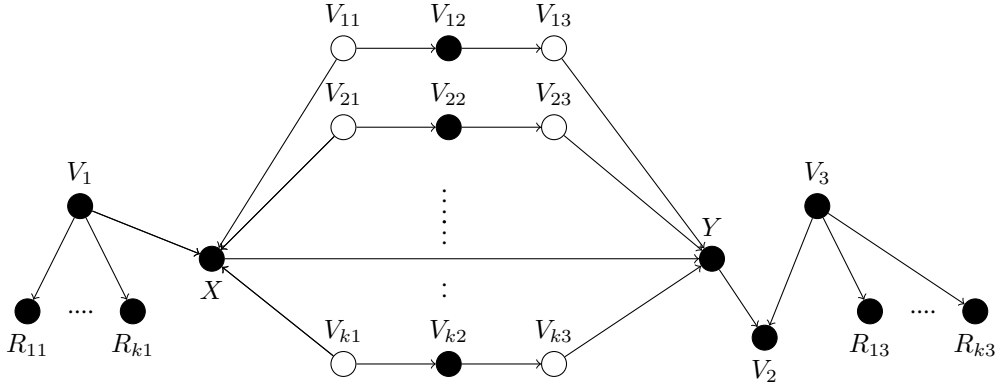
\begin{figure}[H]
		\centering
		\hspace*{2cm}
		\begin{tikzpicture}[scale = 0.7]

		\node[circle,draw,,fill = black,label=below:{$R_{11}$}] (R11) at (1,2) {};
		\node[text width=3cm] at (3.9,2) {....};
		\node[circle,draw,fill = black,label=below:{$R_{k1}$}] (R31) at (3,2) {};

		\node[circle,draw,fill = black,label=above:{$V_1$}] (V1) at (2,4) {};
		\node[circle,draw,fill = black,label=below:{$X$}] (X) at (4.5,3) {};
		\node[circle,draw,,label=above:{$V_{11}$}] (V11) at (7,7) {};
		\node[circle,draw,,fill = black,label=above:{$V_{12}$}] (V12) at  (9,7) {};
		\node[circle,draw,label=above:{$V_{13}$}] (V13) at (11,7)  {};
		\node[circle,draw,label=above:{$V_{21}$}] (V21) at  (7,5.5) {};
		\node[circle,draw,,fill = black,label=above:{$V_{22}$}] (V22) at  (9,5.5) {};
		\node[circle,draw,label=above:{$V_{23}$}] (V23) at  (11,5.5)  {};
		\node[circle,draw,fill = black,label=above:{$Y$}] (Y) at  (14,3) {};
		\node[circle,draw,fill = black,label=below:{$V_2$}] (V3) at  (15,1.5) {};
		\node[circle,draw,fill = black,label=above:{$V_3$}] (V4) at  (16,4) {};
		\node[circle,draw,,fill = black,label=below:{$R_{13}$}] (R13) at (17,2) {};
		\node[text width=3cm] at (19.9,2) {....};	
		\node[circle,draw,,fill = black,label=below:{$R_{k3}$}] (R33) at (19,2) {};		
		\node[text width=3cm] at (11,4.3) {.};
		\node[text width=3cm] at (11,4.1) {.};
		\node[text width=3cm] at (11,3.9) {.};
		\node[text width=3cm] at (11,3.7) {.};
		\node[text width=3cm] at (11,3.5) {.};
		\node[text width=3cm] at (11,3.3) {.};
		\node[text width=3cm] at (11,3.7) {.};
		\node[text width=3cm] at (11,2.5) {.};
		\node[text width=3cm] at (11,2.3) {.};
		\node[circle,draw,label=above:{$V_{k1}$}] (Vk1) at  (7,1)   {};
		\node[circle,draw,,fill = black,label=above:{$V_{k2}$}] (Vk2) at  (9,1) {};
		\node[circle,draw,label=above:{$V_{k3}$}] (Vk3) at  (11,1)  {};

		\draw[->] (V1)-- (R11);
		\draw[->] (V1)-- (R31);
		\draw[->] (V1)-- (X);

		\draw[->] (V11)-- (X);
		\draw[->] (V11)-- (V12);
		\draw[->] (V12)-- (V13);
		\draw[->] (V1)-- (X);
		\draw[->] (V13)-- (Y);
		\draw[->] (Y)-- (V3);
		\draw[->] (V4)-- (V3);
		\draw[->] (V4)-- (R13);
		\draw[->] (V4)-- (R33);

		\draw[->] (V21) -- (X);
		\draw[->] (V21)-- (V22);
		\draw[->] (V22)-- (V23);
		\draw[->] (V21)-- (X);
		\draw[->] (V23)-- (Y);

		\draw[->] (Vk1) -- (X);
		\draw[->] (Vk1)-- (Vk2);
		\draw[->] (Vk2)-- (Vk3);
		\draw[->] (Vk1)-- (X);
		\draw[->] (Vk3)-- (Y);
		\draw[->] (X)-- (Y);

		\end{tikzpicture}
		\centering
		\caption[loftitle]{An example of exponential number of m-adjustment sets}  
		\label{Fig5}              
	\end{figure}
	
	The above example demonstrates that any algorithm that lists all m-adjustment sets will be exponential time complexity. To deal with this issue, we will provide an algorithm with polynomial delay complexity \cite{takata2010space}. Polynomial delay algorithms require polynomial time to generate the first output (or indicate failure) and the time between any two consecutive outputs is polynomial as well.
	
	To facilitate the construction of a listing algorithm, we introduce a graph transformation called \emph{Proper Backdoor Graph} originally introduced by Van der Zander, Liskiewicz, and Textor (2014).
\begin{defn}[Proper Backdoor Graph \cite{vanderZander:2014:CSA:3020751.3020845}]
\label{def:gpbd}
Let $G$ be a causal graph, and $\tb{X},\tb{Y}$ be disjoint subsets of variables. The proper backdoor graph, denoted as $G_{\tb{X}, \tb{Y}}^{pbd}$, is obtained from $G$ by removing the first edge of every proper causal path from $\tb{X}$ to $\tb{Y}$.
\end{defn}
	
Next we present an alternative equivalent formulation of the m-adjustment criterion in Def.~\ref{defMcriterion} that will be useful in constructing a listing algorithm.
	\begin{defn}[M-Adjustment Criterion, Math. Version]\label{defMathcriterion}
Given a m-graph $G$ over observed variables $\tb{V}=\tb{V}_o\cup \tb{V}_m$ and missingness indicators $\tb{R}$, and disjoint sets of variables $\tb{X, Y, Z}\subseteq \tb{V}$, letting $\tb{W} = \tb{V}_m \cap (\tb{X}\cup\tb{Y}\cup\tb{Z})$, $\tb{Z}$ satisfies the m-adjustment criterion relative to the pair ($\tb{X}, \tb{Y}$) if  
	\begin{enumerate}[a)]
		\item $\tb{Z} \cap Dpcp(\tb{X}, \tb{Y}) = \phi$
		\item $(\tb{Y} \independent \tb{X} \mid \tb{Z}, \tb{R}_{\tb{W}})_{G_{\tb{X}, \tb{Y}}^{pbd}}$
		\item $(\tb{Y} \independent \tb{R}_{\tb{W}} \mid \tb{X})_{G_{\overline{X}}}$
		\item  $( (\tb{X} \cap An(\tb{R}_{\tb{W}})) \independent \tb{Y})_{G_{\underline{X}}}$ 
	\end{enumerate}
    where 
    \begin{align}
    D_{pcp}(\tb{X},\tb{Y})=De ((De(\tb{X})_{G_{\overline{X}}}\setminus \tb{X} ) \cap An(\tb{Y})_{G_{\overline{X}}}).
    \end{align}
	\end{defn}
In 	Definition~\ref{defMathcriterion}, $D_{pcp}(\tb{X},\tb{Y})$, originally introduced in \cite{vanderZander:2014:CSA:3020751.3020845}, represents the set of descendants of those variables in a proper causal path from $\tb{X}$ to $\tb{Y}$.
	
	\begin{prop}\label{prop-eqv-Mcrtn-Math} Definition \ref{defMathcriterion} and Definition \ref{defMcriterion} are equivalent.
	\end{prop}
Note that all the proofs of the propositions and  theorems in Section~\ref{sec-list} are given in Appendix A.

Finally to help understanding the logic of the algorithm we introduce a definition originally introduced in \cite{correa2018generalized}:
	\begin{defn}[Family of Separators \cite{correa2018generalized}]\label{defFSeprator} For a disjoint set of variables \tb{X}, \tb{Y}, \tb{E} and $\tb{I} \subseteq \tb{E}$, a family of separators is defined as follows:
	\begin{equation}
	{Z}_{G(\tb{X},\tb{Y})} \langle \tb{I}, \tb{E} \rangle := \{\tb{Z} \mid ( \tb{X} \independent \tb{Y} \mid \tb{Z} )_G \mbox{ and } \tb{I} \subseteq \tb{Z} \subseteq \tb{E} \},
	\end{equation}
	which represent the set of all sets that d-separate \tb{X} and \tb{Y} and encompass all variables in set \tb{I} but do not have any variables outside \tb{E}.
	\end{defn}

Algorithm~\ref{alg:1} presents the function ListMAdj that lists all the m-adjustment sets in a given m-graph $G$ for recovering the causal effect of \tb{X} on \tb{Y}. We note that the algorithm uses an external function FindSep described in \cite{vanderZander:2014:CSA:3020751.3020845} (not presented in this paper). FindSep(G, \tb{X}, \tb{Y}, \tb{I}, \tb{E}) will return a set in ${Z}_{G(\tb{X}, \tb{Y})} \langle \tb{I}, \tb{E} \rangle$ if such a set exists; otherwise it returns $\bot$ representing failure. 

\begin{algorithm}
	\caption{Listing all the m-adjustment sets}\label{alg:1}
	\DontPrintSemicolon
	\SetAlgoNoLine
	\Fn{ListMAdj ($G,\tb{X},\tb{Y},\tb{V}_o, \tb{V}_m,\tb{R}$) }{ 
		$G^{pbd}_{\tb{X},\tb{Y}} \gets$ compute proper back-door graph G\;
		$\tb{E} \gets (\tb{V}_o \cup \tb{V}_m \cup \tb{R}) \setminus \{ \tb{X} \cup \tb{Y} \cup D_{pcp}(\tb{X},\tb{Y})\}.$\; 
		
		ListSepConditions($G^{pbd}_{\tb{X},\tb{Y}},\tb{X},\tb{Y},\tb{R},\tb{V}_o, \tb{V}_m, \tb{I} = \{\tb{R}_{\tb{X} \cap \tb{V}_m} \cup \tb{R}_{\tb{Y}\cap \tb{V}_m} \}, \tb{E}$)}
	\Fn{ListSepConditions ($G,\tb{X},\tb{Y},\tb{R},\tb{V}_o, \tb{V}_m,\tb{I},\tb{E}$)}{
		\If{$(\tb{Y} \independent \tb{R}_{{\tb{I}}} \mid \tb{X})_{G_{\overline{\tb{X}}}}$ and $((\tb{X} \cap An(\tb{R}_{{\tb{I}}})) \independent \tb{Y})_{G_{\underline{\tb{X}}}}$ and $FindSep(G,\tb{X},\tb{Y},\tb{I}, \tb{E}) \neq \bot$ }{
			\eIf{$\tb{I} = \tb{E}$}{
				Output($\tb{I} \setminus \tb{R})$\;}{
				W $\gets$ arbitrary variable from $\tb{E} \setminus (\tb{I} \cup \tb{R})$\; 
				\If{$W \in \tb{V}_o$}{
					ListSepConditions($G,\tb{X},\tb{Y},\tb{R},\tb{V}_o, \tb{V}_m,$
					$\tb{I} \cup \{W\}, \tb{E}$)\;	
					ListSepConditions($G,\tb{X},\tb{Y},\tb{R},\tb{V}_o, \tb{V}_m, \tb{I}$
					$ , \tb{E} \setminus \{W\}$)\;	
				}				\If{$W \in \tb{V}_m$ and $\tb{R}_{W} \in \tb{E}$}{ 
					ListSepConditions($G,\tb{X},\tb{Y},\tb{R},\tb{V}_o, \tb{V}_m,$
					$\tb{I} \cup \{W, \tb{R}_{W}\}, \tb{E}$)\;	
					ListSepConditions($G,\tb{X},\tb{Y},\tb{R},\tb{V}_o, \tb{V}_m, \tb{I}$, 
					 $\tb{E} \setminus \{W, \tb{R}_{W}\}$)\;	
			}}
	}}
\end{algorithm}

 Function ListMAdj works by first excluding all variables lying in the proper causal paths from consideration (Line 3) and then calling the function ListSepConditions (Line 4) to return all the m-adjustment sets. The function of ListSepConditions is summarized in the following proposition:
	\begin{prop}[Correctness of ListSepConditions]\label{thmCorAlgListCond} 
		Given a m-graph $G$ 
		and sets of disjoint variables \tb{X}, \tb{\tb{Y}}, and \tb{E} and $\tb{I} \subseteq \tb{E}$,  ListSepConditions lists all \tb{Z} variables such that: \\
		
		$\tb{Z} \in \{\tb{Z} \mid (\tb{X} \independent \tb{Y} \mid \tb{Z},\tb{R}_{\tb{Z}},\tb{R}_{\tb{X} \cap \tb{V}_{m}},\tb{R}_{\tb{Y} \cap \tb{V}_{m}})_{G_{\tb{X}, \tb{Y}}^{pbd}} \mbox{ }\& \mbox{ }(\tb{Y} \independent \tb{R}_{\tb{Z}} \mid \tb{X})_{G_{\overline{X}}}\mbox{ } \& \mbox{ }( (\tb{X} \cap An(\tb{R}_{\tb{Z}})) \independent \tb{Y})_{G_{\underline{X}}}\& \mbox{ }\tb{I} \subseteq \tb{Z} \subseteq \tb{E} \}$
        
        Where $\tb{R}_{\tb{Z}}$ is a shorthand for $\tb{R}_{\tb{Z}\cap \tb{V}_m}$. 
	\end{prop}

    ListSepConditions, by considering both including and not including each variable, recursively generates all subset of variables in \tb{V} and for each generated set examines whether the conditions (b), (c), and (d) in Def.~\ref{defMathcriterion} holds or not. If those conditions were satisfied, the algorithm will return that candidate set as a m-adjustment set. ListSepConditions generates each potential set by taking advantage of back-track algorithm and at each recursion for a variable $W \in \tb{V}$ examines two cases of having $W$ in candidate set or not. If $W \in \tb{V}_o$, then the algorithm examines having and not having this variable in the m-adjustment set and continues to decide about the rest of the variables in next recursion. If $W \in \tb{V}_m$, then the algorithm includes both $W$ and $R_W$ in the candidate m-adjustment set. Therefore, the algorithm considers both cases of having $ W , R_W$ and not having them in the candidate set. ListSepConditions, at the beginning of each recursion in Line 7, examines whether the candidate m-adjustment set so far satisfies the conditions (b), (c), (d) in Def.~\ref{defMcriterion} or not. If any of them is not satisfied, the recursion stops for that candidate set. The function FindSep examines the existence of a set containing all variables in \tb{I} and not having any of $\tb{V} \setminus \tb{E}$ that d-separates \tb{X} from \tb{Y}. If this set does not exist FindSep returns {$ \bot $}. ListSepConditions utilizes FindSep in order to check satisfaction of condition (b) in Def.~\ref{defMathcriterion} for the candidate set. Since the graph $G$ that is given to FindSep is a proper back-door graph, all paths between \tb{X} and \tb{Y} in this graph is non-causal. Therefore, if a set separates $\tb{X}$ and $\tb{Y}$ in $G^{pbd}$, this set blocks all non-causal paths from $\tb{X}$ to $\tb{Y}$ in $G$.
    
 The following theorem states that ListMAdj lists all the m-adjustment sets in a given m-graph $G$ for recovering the causal effect of \tb{X} on \tb{Y}.  
  		\begin{thm}[Correctness of ListMAdj]\label{thmCorAlgListAdj}
			Given a m-graph $G$ and sets of disjoint variables \tb{X}, \tb{Y}, ListMAdj returns all the sets that satisfy the m-adjustment criterion relative to (\tb{X},\tb{Y}).
		\end{thm}
  
The follow results state that Algorithm~\ref{alg:1} is polynomial delay.		
		\begin{prop}[Complexity of ListSepConditions]\label{thmCpxAlg}
			ListSepConditions for a given graph $G $ has a time complexity of $O(n(n+m))$ polynomial delay where $n$ and $m$ are the number of variables and edges in $G$ respectively.
		\end{prop}
		\begin{thm}[Complexity of ListMAdj]\label{thmListAdj}
			ListMAdj for a given graph $G$ returns all the m-adjustment sets with $O(n(n+m))$ polynomial delay where $n$ and $m$ are the number of variables and edges in $G$ respectively. 
		\end{thm}

		\subsection{Finding Minimum M-Adjustment Set}
		
		The problem of finding a m-adjustment set with minimum number of variables is important from several aspects. This can reduce the computational time, while making the result more interpretive. The cost of collecting more variables might be another reason researchers prefer to find a minimum set. Next we present an algorithm that for a given graph $G$ with disjoint sets \tb{X} and \tb{Y} returns a m-adjustment set with the minimum number of variables.

		Function FindMinAdjSet takes a m-graph $G$ as input and returns a m-adjustment set with minimum number of variables. The function works by first removing all variables that violate Conditions (a), (c), and (d) in the m-adjustment criterion Def.~\ref{defMcriterion} in lines 2 to 5, and then calling an external function FinMinCostSep given in \cite{vanderZander:2014:CSA:3020751.3020845} which returns a minimum weight separator. FindMinAdjSet sets all the weights for each variable to be 1 to get a set with minimum size.   
		


		

		\begin{algorithm}
			\caption{Find minimum size m-adjustment set}\label{alg:2}
			\DontPrintSemicolon
			\SetAlgoNoLine
			\Fn{FindMinAdjSet($G,\tb{X},\tb{Y},\tb{V}_o, \tb{V}_m,\tb{R}$)}{ 
			     G\textprime $\gets$ compute proper 
			    back-door graph $G^{pbd}_{X,Y}$\;
				$\tb{E} \gets (\tb{V}_o \cup \tb{V}_m) \setminus \{ \tb{X} \cup \tb{Y} \cup D_{pcp}(\tb{X},\tb{Y})\}.$\; 
				$\tb{E}\textprime \gets \{E \in \tb{E} \mid E \in \tb{V}_o \mbox{ or } E \in \tb{V}_m \mbox{ and } ( \tb{R}_E \independent \tb{Y} \mid \tb{X})_{G\textprime_{\overline{X}}} \}$\;
				
				$\tb{E}\textprime\textprime \gets \{E \in \tb{E}\textprime \mid \tb{E} \in \tb{V}_o \mbox{ or } E \in \tb{V}_m \mbox{ and } ( \tb{X} \cap An(\tb{R}_{\tb{E}}) \independent \tb{Y})_{G\textprime_{\underline{X}}}\}$\;
				$\tb{W} \gets$ 1 for all variables \;
				$\tb{I} \gets$ empty set \;
				\tb{N} $\gets$ FindMinCostSep($G$\textprime, \tb{X}, \tb{Y},     \tb{I}, \tb{E}\textprime\textprime, \tb{W})\;
				return \tb{N} $\cup$ $\tb{R}_N$
			}
		\end{algorithm}
		
		\begin{thm}[Correctness of FindMinAdjSet]\label{thmCorAlgListMin}
			Given a m-graph $G$ and disjoint sets of variables \tb{X}, \tb{Y}, FindMinAdjSet returns a m-adjustment set relative to (\tb{X},\tb{Y}) with minimum number of variables.
		\end{thm}
		
		\begin{thm}[Time Complexity of FindMinAdjSet] \label{propTimeComListMin}
			FindMinAdjSet has a time complexity of O($n^3$).
		\end{thm}

\section{Adjustment from both Selection Bias and Missing Data}

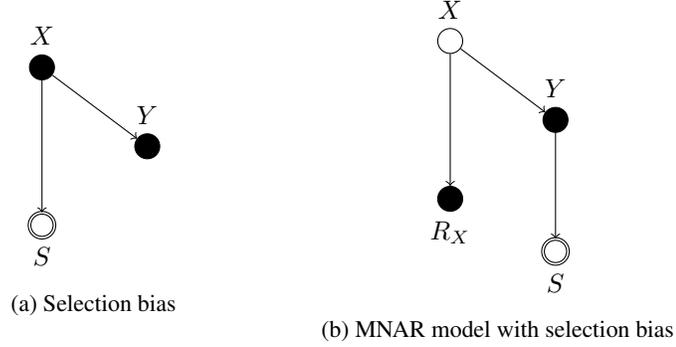
\begin{figure}
	\begin{center}
		\begin{tabular}{cc}
			\begin{subfigure}[normal]{0.3\linewidth}
			\centering
			\begin{tikzpicture}[scale = 0.7]
			\node[circle,draw,fill = black,label=above:{$X$}] (X) at (12,5) {};
			\node[circle,draw,,fill = black, ,label=above:{$Y$}] (Y) at (14,3.5) {};
			\node[circle,draw, ,double ,label=below:{$S$}] (S) at (12,2) {};
			
			\draw[->] (X) -- (Y);
			\draw[->] (X) -- (S);
			\end{tikzpicture}
			\caption{Selection bias} 
		\end{subfigure}
			&
			\begin{subfigure}[normal]{0.3\linewidth}
			\centering
				\begin{tikzpicture}[scale = 0.7]
				\node[circle,draw,label=above:{$X$}] (X) at (12,5) {};
				\node[circle,draw,,fill = black, ,label=above:{$Y$}] (Y) at (14,3.5) {};
				\node[circle,draw, ,fill = black, ,label=below:{$R_X$}] (RX) at (12,2) {};
				\node[circle,draw,double ,label=below:{$S$}] (S) at (14,1) {};
				
				\draw[->] (X) -- (Y);
				\draw[->] (X) -- (RX);
				\draw[->] (Y) -- (S);
				\end{tikzpicture}
				\caption{MNAR model with selection bias} 
			\end{subfigure}
		\end{tabular}
		\caption{Examples of selection bias and MNAR}
		\label{fig-6}
	\end{center}
\end{figure}

In Sections~\ref{sec-missing} and~\ref{sec-list} we have addressed the task of recovering causal effects by adjustment from missing data. In practice another common issue that data scientists face in estimating causal effects is selection bias. Selection bias can be modeled by introducing a binary indicator variable $S$ such that $S=1$ if a unit is included in the sample, and
$S=0$ otherwise \cite{bareinboim2014recovering}. Graphically selection bias is modeled by a special hollow node $S$ (drawn round with double border) that is pointed to by every variable in $\tb{V}$ that affects the process by which an unit is included in the data. In Fig.~\ref{fig-6}(a), for example, selection is affected by the treatment variable.

In the context of selection bias, the observed distribution is $P(\tb{V}\mid S=1)$, collected under seletion bias, instead of $P(\tb{V})$. The goal of inference is to recover the causal effect $P(\tb{y} \mid do(\tb{x}))$ from $P(\tb{V}\mid S=1)$. The use of adjustment for recovering causal effects in this setting has been studied and complete adjustment conditions have been developed in  \cite{correa2017causal,correa2018generalized}. 
\begin{table}
\centering
\caption{ An example of data compatible with Fig.\ref{fig-6}}
\scalebox{1}{
\begin{tabular}{|c @{\hskip 0.3in} c @{\hskip 0.3in}c@{\hskip 0.3in} c@{\hskip 0.3in} c@{\hskip 0.2in} |} 
 \hline
 ID & X & $R_X$ & Y & S\\   
 \hline\hline 
 1 &1 & 1 & 0 & 1 \\ 
 \hline
 2 &0 & 1 & 1 & 1 \\
 \hline
 3 &NA & NA & NA & 0\\
 \hline
 4 &NA & 0 & 1 & 1 \\
 \hline
 5 &NA & NA & NA & 0 \\ 
 \hline
\end{tabular}}

\label{tab-1}
\end{table}

What if the observed data suffer from both selection bias and missing values? In the model in Fig.~\ref{fig-6}(b), for example, whether a unit is included in the sample depends on the value of the outcome. If a unit is included in the sample, the values of treatment $X$ could be missing depending on the actual $X$ values. Table.~\ref{tab-1} indicate a compatible examples with the Fig.~\ref{fig-6}(b) declaring the difference of missing and selection mechanism. To the best of our knowledge, causal inference under this setting has not been formally studied.

In this section, we will characterize the use of adjustment for causal effect identification when the observed data suffer from both selection bias and missing values. First we introduce an adjustment formula called \emph{MS-adjustment} for recovering causal effect under both missing data and selection bias. Then we provide a complete condition under which a set \tb{Z} is valid as MS-adjustment set. We then provide an example to demonstrate its application.

\begin{defn}[MS-Adjustment Formula] Given a m-graph $G$ over observed variables $\tb{V}=\tb{V}_o\cup \tb{V}_m$ and missingness indicators $\tb{R}$ augmented with a selection bias indicator $S$, a set $\tb{Z}\subseteq \tb{V}$ is called a ms-adjustment (adjustment under missing data and selection bias) set for estimating the causal effect of \tb{X} on \tb{Y}, if for every model compatible with $G$ it holds that
	\begin{align}\label{eq-ms-adj}
	P(\tb{y} \mid do(\tb{x})) = \sum_{\tb{z}} P(\tb{y} \mid \tb{x}, \tb{z} ,\tb{R}_{\tb{W}} = 1, S = 1)P(\tb{z} \mid \tb{R}_{\tb{W}} = 1, S = 1),
	\end{align}
	where $\tb{W} = \tb{V}_m \cap (\tb{X}\cup\tb{Y}\cup\tb{Z})$.
\end{defn}
Both terms on the right-hand-side of Eq.~(\ref{eq-ms-adj}) are recoverable from selection biased data in which all variables in $\tb{X}\cup \tb{Y}\cup \tb{Z}$ are observed. Therefore the causal effect 	$P(\tb{y} \mid do(\tb{x}))$ is recoverable if it can be expressed in the form of ms-adjustment.

Next we provide a complete criterion to determine whether a set \tb{Z} is an admissible ms-adjustment. 
	\begin{defn}[MS-Adjustment Criterion]\label{defMScriterion} Given a m-graph $G$ over observed variables $\tb{V}=\tb{V}_o\cup \tb{V}_m$ and missingness indicators $\tb{R}$ augmented with a selection bias indicator $S$, and disjoint sets of variables $\tb{X, Y, Z}$, letting $\tb{W} = \tb{V}_m \cap (\tb{X}\cup\tb{Y}\cup\tb{Z})$, $\tb{Z}$ satisfies the ms-adjustment criterion relative to the pair ($\tb{X}, \tb{Y}$) if  
	\end{defn}

\begin{enumerate}[a)]
		\item 
		No element of \tb{Z} is a descendant in $G_{\overline{X}}$ of any $W \notin \tb{X} $ which lies on a proper causal path from \tb{X} to \tb{Y}. 
		\item
		All non-causal paths between \tb{X} and \tb{Y} in $G$ are blocked by \tb{Z}, $\tb{R}_{\tb{W}}$, and S. 	
		\item 
		$\tb{R}_{\tb{W}} \mbox{ and } S$ are d-separated from \tb{Y} given \tb{X} under the intervention of do(\tb{x}). i.e., $(\tb{Y} \independent (\tb{R}_{\tb{W}} \cup S) \mid \tb{X})_{G_{\overline{X}}}$	
		\item 
		Every $X \in \tb{X} $ is either a non-ancestor of $\{ \tb{R}_{\tb{W}}, S \}$ or it is d-separated from \tb{Y} in $G_{\underline{X}}$. i.e., $\forall X \in \tb{X} \cap An(\tb{R}_{\tb{W}} \cup S), ( X \independent \tb{Y})_{G_{\underline{X}}}$  
	\end{enumerate}

	\begin{thm}[MS-Adjustment] \label{thm-MSadj-com} A set \tb{Z} is a ms-adjustment set for recovering causal effect of \tb{X} on \tb{Y} by the ms-adjustment formula in Definition~\ref{eq-ms-adj} if and only if it satisfies the ms-adjustment criterion in Definition~\ref{defMScriterion}.  
	\end{thm}
The proof of Theorem~\ref{thm-MSadj-com} is presented in the Appendix B.

	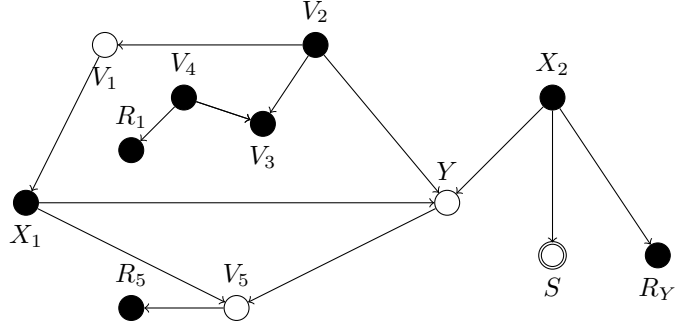
\begin{figure}[H]
		\centering
		\begin{tikzpicture}[scale = 0.7]

		\node[circle,draw,,fill = black,label=below:{$X_1$}] (X1) at (1,4) {};
		\node[circle,draw,label=below:{$V_1$}] (V1) at (2.5,7) {};

		\node[circle,draw,fill = black,label=above:{$V_2$}] (V2) at (6.5,7) {};
		\node[circle,draw,fill = black,label=below:{$V_3$}] (V3) at  (5.5,5.5) {};
		\node[circle,draw,fill = black,label=above:{$V_4$}] (V4) at  (4,6)   {};
		\node[circle,draw,fill = black,label=above:{$R_1$}] (R1) at  (3,5)  {};
		\node[circle,draw,label=above:{$V_5$}] (V5) at  (5,2) {};
		\node[circle,draw,fill = black,label=above:{$R_5$}] (R5) at  (3,2) {};

		\node[circle,draw,label=above:{$Y$}] (Y) at  (9,4)  {};
		\node[circle,draw,fill = black,label=above:{$X_2$}] (X2) at  (11,6) {};
		\node[circle,draw,double,label=below:{$S$}] (S) at  (11,3) {};
		\node[circle,draw,fill = black,label=below:{$R_Y$}] (RY) at  (13,3) {};

		\draw[->] (V1)-- (X1);
		\draw[->] (V2)-- (V1);
		\draw[->] (V2)-- (V3);

		\draw[->] (V4)-- (V3);
		\draw[->] (V4)-- (R1);
		\draw[->] (V2)-- (Y);
		\draw[->] (X1)-- (Y);
		\draw[->] (X2)-- (Y);
		\draw[->] (X2)-- (S);
		\draw[->] (X2)-- (RY);
		\draw[->] (V4)-- (V3);
		\draw[->] (X1)-- (V5);
		\draw[->] (Y)-- (V5);
		\draw[->] (V5)-- (R5);
		
		\end{tikzpicture}
		\centering
		\caption[loftitle]{An example for recovering causal effect under both selection bias and MNAR data}  
		\label{Fig6}              
	\end{figure}

    To demonstrate the application of Theorem~\ref{thm-MSadj-com}, consider the causal graph in Fig.~\ref{Fig6} where $V_1$ ,$V_5$, $Y$ may have missing values and the selection $S$ depends on the values of $X_2$. To recover the causal effect of $\{X_1, X_2\}$ on variable $Y$, $V_1$ satisfies the ms-adjustment criterion. To confirm we derive using do-calculus as follows: 
	\begin{align}
	&P(y \mid do(x_1, x_2)) \nonumber
	\\&= P(y \mid do(x_1, x_2), S = 1, R_y = 1 , R_1 =1) 
	\\&= \sum_{V_{1}}P(y \mid do(x_1, x_2),V_1,S = 1, R_y = 1 , R_1 =1)P( V_1 \mid do(x_1, x_2), S = 1, R_y = 1 , R_1 =1 )
	\\&= \sum_{V_1}P(y \mid do(x_1, x_2), V_1 ,S = 1, R_y = 1 , R_1 =1)P(V_1 \mid S = 1, R_y = 1 , R_1 =1)
	\\& = \sum_{V_1}P(y \mid x_1, x_2 ,V_1, S = 1, R_y = 1 , R_1 =1)P(V_1 \mid S = 1, R_y = 1 , R_1 =1)
	\end{align}

   We note that the two algorithms given in Section \ref{sec-list}, for listing all m-adjustment sets and finding a minimum size m-adjustment set, can be extended to list all ms-adjustment sets and find a minimum ms-adjustment set with minor modifications.

		\section{Conclusion}
		In this paper we introduce a m-adjustment formula for recovering causal effect in the presence of MNAR data and provide a necessary and sufficient graphical condition - m-adjustment criterion for when a set of covariates are valid m-adjustment.  
		We introduce a ms-adjustment formulation for causal effects identification in the presence of both selection bias and MNAR data and provide a necessary and sufficient graphical condition - ms-adjustment criterion for when a set of covariates are valid ms-adjustment. 
		We develop an algorithm that lists all valid m-adjustment or ms-adjustment sets in polynomial delay time, and an algorithm that finds a valid m-adjustment or ms-adjustment set containing the minimum number of variables. The algorithms are useful for data scientists to select adjustment sets with desired properties (e.g. low measurement cost).  
 Adjustment is the most used tool for estimating causal effect in the data sciences. The results in this paper should help to alleviate the problem of missing data and selection bias in a broad range of data-intensive applications. 		

\section*{Acknowledgements}

This research was partially supported by NSF grant IIS-1704352 and ONR grant N000141712140.

%

\bibliographystyle{splncs04}
\bibliography{References}

\begin{thebibliography}{10}
\providecommand{\url}[1]{\texttt{#1}}
\providecommand{\urlprefix}{URL }
\providecommand{\doi}[1]{https://doi.org/#1}

\bibitem{DBLP:conf/aaai/BareinboimT15}
Bareinboim, E., Tian, J.: Recovering causal effects from selection bias. In:
  Proceedings of the Twenty-Ninth {AAAI} Conference on Artificial Intelligence.
  pp. 3475--3481 (2015)

\bibitem{bareinboim2014recovering}
Bareinboim, E., Tian, J., Pearl, J.: Recovering from selection bias in causal
  and statistical inference. In: Proceeding of the Twenty-Eighth AAAI
  Conference on Artificial Intelligence. pp. 2410--2416 (2014)

\bibitem{correa2017causal}
Correa, J.D., Bareinboim, E.: Causal effect identification by adjustment under
  confounding and selection biases. In: Proceedings of the Thirty-First AAAI
  Conference on Artificial Intelligence. pp. 3740--3746 (2017)

\bibitem{correa2018generalized}
Correa, J.D., Tian, J., Bareinboim, E.: Generalized adjustment under
  confounding and selection biases. In: Thirty-Second AAAI Conference on
  Artificial Intelligence. pp. 6335--6342 (2018)

\bibitem{Huang:2006:ICB:1597348.1597371}
Huang, Y., Valtorta, M.: Identifiability in causal bayesian networks: A sound
  and complete algorithm. In: Proceedings of the 21st National Conference on
  Artificial Intelligence. vol.~2, pp. 1149--1154. AAAI Press (2006)

\bibitem{koller2009probabilistic}
Koller, D., Friedman, N., Bach, F.: Probabilistic graphical models: principles
  and techniques. MIT press (2009)

\bibitem{Little:1986:SAM:21412}
Little, R.J.A., Rubin, D.B.: Statistical Analysis with Missing Data. John Wiley
  \& Sons, Inc. (1986)

\bibitem{DBLP:conf/uai/MarlinZRS07}
Marlin, B.M., Zemel, R.S., Roweis, S.T., Slaney, M.: Collaborative filtering
  and the missing at random assumption. In: Proceedings of the Twenty-Third
  Conference on Uncertainty in Artificial Intelligence. pp. 267--275 (2007)

\bibitem{DBLP:conf/ijcai/MarlinZRS11}
Marlin, B.M., Zemel, R.S., Roweis, S.T., Slaney, M.: Recommender systems,
  missing data and statistical model estimation. In: Proceedings of the 22nd
  International Joint Conference on Artificial Intelligence. pp. 2686--2691
  (2011)

\bibitem{mohan2014graphical}
Mohan, K., Pearl, J.: Graphical models for recovering probabilistic and causal
  queries from missing data. In: Advances in Neural Information Processing
  Systems. pp. 1520--1528 (2014)

\bibitem{mohan2013graphical}
Mohan, K., Pearl, J., Tian, J.: Graphical models for inference with missing
  data. In: Advances in neural information processing systems. pp. 1277--1285
  (2013)

\bibitem{pearl1995causal}
Pearl, J.: Causal diagrams for empirical research. Biometrika  \textbf{82}(4),
  669--688 (1995)

\bibitem{Pearl:2009:CMR:1642718}
Pearl, J.: Causality: Models, Reasoning and Inference. Cambridge University
  Press, 2nd edn. (2009)

\bibitem{Pearl2016CausalPrimer}
Pearl, J., Glymour, M., Jewell, N.P.: {Causal inference in statistics: A
  Primer}. John Wiley {\&} Sons (2016)

\bibitem{perkovic2017complete}
Perkovic, E., Textor, J., Kalisch, M., Maathuis, M.H.: Complete graphical
  characterization and construction of adjustment sets in markov equivalence
  classes of ancestral graphs. The Journal of Machine Learning Research
  \textbf{18}(1),  8132--8193 (2017)

\bibitem{robins2000marginal}
Robins, J.M., Hernan, M.A., Brumback, B.: Marginal structural models and causal
  inference in epidemiology. Epidemiology  \textbf{11}(5) (2000)

\bibitem{rubin1976inference}
Rubin, D.: Inference and missing data. Biometrika  \textbf{63}(3),  581--592
  (1976)

\bibitem{rubin1978multiple}
Rubin, D.B.: Multiple imputations in sample surveys-a phenomenological bayesian
  approach to nonresponse. In: Proceedings of the survey research methods
  section of the American Statistical Association. vol.~1, pp. 20--34 (1978)

\bibitem{Shpitser:2015:MDC:3020847.3020930}
Shpitser, I., Mohan, K., Pearl, J.: Missing data as a causal and probabilistic
  problem. In: Proceedings of the Thirty-First Conference on Uncertainty in
  Artificial Intelligence. pp. 802--811 (2015)

\bibitem{shpitser2006identification}
Shpitser, I., Pearl, J.: Identification of joint interventional distributions
  in recursive semi-markovian causal models. In: Proceedings of the National
  Conference on Artificial Intelligence. vol.~21, p.~1219 (2006)

\bibitem{Shpitser:2010:VCA:3023549.3023612}
Shpitser, I., VanderWeele, T., Robins, J.M.: On the validity of covariate
  adjustment for estimating causal effects. In: Proceedings of the Twenty-Sixth
  Conference on Uncertainty in Artificial Intelligence. pp. 527--536. AUAI
  Press (2010)

\bibitem{takata2010space}
Takata, K.: Space-optimal, backtracking algorithms to list the minimal vertex
  separators of a graph. Discrete Applied Mathematics  \textbf{158}(15),
  1660--1667 (2010)

\bibitem{tian2017recovering}
Tian, J.: Recovering probability distributions from missing data. In:
  Proceedings of the Ninth Asian Conference on Machine Learning. vol. PMLR 77
  (2017)

\bibitem{tian2002general}
Tian, J., Pearl, J.: A general identification condition for causal effects. In:
  Eighteenth National Conference on Artificial Intelligence. pp. 567--573
  (2002)

\bibitem{vanderZander:2014:CSA:3020751.3020845}
van~der Zander, B., Li\'{s}kiewicz, M., Textor, J.: Constructing separators and
  adjustment sets in ancestral graphs. In: Proceedings of the Thirtieth
  Conference on Uncertainty in Artificial Intelligence. pp. 907--916. AUAI
  Press (2014)

\end{thebibliography}
\newpage

\section*{Appendix A: Proofs in Section 4}
\tb{Proposition} \tb{\ref{prop-eqv-Mcrtn-Math}}. Definition \ref{defMathcriterion} and Definition \ref{defMcriterion} are equivalent. \\
\noindent\tb{Proof}: Condition (c) and (d) in both definitions are the same. Condition (a) in Def.~\ref{defMathcriterion} indicates that \tb{Z} cannot be in Dpcp(\tb{X},\tb{Y}). i.e., \tb{Z} may not be descendant of any variables lies in proper causal path from \tb{X} to \tb{Y}. This is as same as condition (a) in Def.~\ref{defMcriterion}. In order to prove Def.~\ref{defMathcriterion} $\rightarrow$ Def.~\ref{defMcriterion}, it is left to show Def.~\ref{defMathcriterion} leads to condition (b) in Def.\ref{defMcriterion}. By contradiction, assume there is a open non-causal path from a $X\textprime \in \tb{X}$ to $Y\textprime \in \tb{Y}$. Condition (b) in Def.~\ref{defMathcriterion} requires all non-causal proper back-door paths to be blocked. Therefore, This open non-causal path $p$ does not belongs to proper back-door graph. The path $p$ has edges coming out of \tb{X} and belongs to a proper path $q$. Without lose of generality, assume $X\textprime \in \tb{X}$ is the first and only variable in \tb{X} that lies in the path $p$, otherwise, consider part of the path $p$ with only $X\textprime$ at the beginning of it. Let $W$ be the variable on the other side of this edge, and $Y\textprime \in \tb{Y}$ be the last variable in path $p$ and $Y''$ be the last one in $q$. Path $p$ cannot be a direct path from $X\textprime$ to $Y\textprime$ since it is a non-causal path. Therefore, $p$ should have colliders belong to $\tb{Z} \cup \tb{R}_{\tb{W}}$. These colliders cannot belong to \tb{Z} due to condition (a) in Def.~\ref{defMathcriterion}. Consequently, they should belong to $\tb{R}_{\tb{W}}$ which violates condition (c). Therefore, our assumption about the existence of the path $p$ is not true. For the other direction, the closeness of all non-causal paths by $\tb{Z} \cup \tb{R}_{\tb{W}}$ leads to closeness of non-causal proper back-door paths. To prove this theorem, we used some achievements in \cite{correa2018generalized}. \\
\\\tb{Proposition} \tb{\ref{thmCorAlgListCond}} \tb{(Correctness of ListSepCondition)}.
	Given a m-graph $G$ 
		and sets of disjoint variables \tb{X}, \tb{\tb{Y}}, and \tb{E} and $\tb{I} \subseteq \tb{E}$,  ListSepConditions lists all \tb{Z} variables such that: \\
		
		$\tb{Z} \in \{\tb{Z} \mid (\tb{X} \independent \tb{Y} \mid \tb{Z},\tb{R}_{\tb{Z}},\tb{R}_{\tb{X} \cap \tb{V}_{m}},\tb{R}_{\tb{Y} \cap \tb{V}_{m}})_{G_{\tb{X}, \tb{Y}}^{pbd}} \mbox{ }\& \mbox{ }(\tb{Y} \independent \tb{R}_{\tb{Z}} \mid \tb{X})_{G_{\overline{X}}}\mbox{ } \& \mbox{ }( (\tb{X} \cap An(\tb{R}_{\tb{Z}})) \independent \tb{Y})_{G_{\underline{X}}}\& \mbox{ }\tb{I} \subseteq \tb{Z} \subseteq \tb{E} \}$
        
        Where $\tb{R}_{\tb{Z}}$ is a shorthand for $\tb{R}_{\tb{Z}\cap \tb{V}_m}$. 
\\\tb{Proof}: 
The proof for this theorem includes two parts. In the first part, we prove the algorithm returns sound results, and in the second part we prove the algorithm returns all the correct results. \\
\\ \tb{Part 1}: Line 8 is where the algorithm returns the output. To get to line 8, the conditions in line 7 need to be satisfied. The conditions of $(\tb{Y} \independent \tb{R}_{\tb{Z}}  \mid \tb{X})_{G_{\overline{X}}}$ and $((\tb{X} \cap An(\tb{R}_{\tb{Z}})) \independent \tb{Y})_{G_{\underline{X}}}$ are exactly checked in line 7. We explain how  the algorithm makes sure the condition $(\tb{X} \independent \tb{Y} \mid \tb{Z} , \tb{R}_{\tb{Z}}, \tb{R}_{\tb{X} \cap \tb{V}_{m}},\tb{R}_{\tb{Y} \cap \tb{V}_{m}})_{G_{\tb{X}, \tb{Y}}^{pbd}}$ holds. Function FindSep examines if a candidate set is a valid separator for the sets \tb{X} and \tb{Y} in the graph $G$. Note that in our case, we are giving proper back-door graph as an input to this function. Therefore, all paths from \tb{X} to \tb{Y} are non-causal paths, and a set is a separator relative to the graph $G$ and sets \tb{X} and \tb{Y}, if and only if it closes all non-casual paths. If the set closes all non-causal path, the FindSep function returns true. Therefore, all outputs satisfy the three  m-adjustment criterion (b,c,d). \\
\\\tb{Part 2}: We prove the algorithm returns all sets satisfying m-adjustment criterion (b,c,d). The algorithm examines all subsets of \tb{E} as a candidate sets by checking the two potential sets including and excluding $W \in \tb{E}$ in the sets with a backtracking. After selecting $W$, the algorithm evaluates type of $W$ to see whether it belongs to $\tb{V}_m$ or $\tb{V}_o$. If $W \in \tb{V}_o$, the algorithm goes to the two next recursions of having $W$ in the set and not having it. If $W \in \tb{V}_m$, it ensures to include $R_W$ or not include it along variable $W$. Therefore, we evaluate all subsets of $\tb{E}$. It is only necessary for the algorithm to ensure not abort any recursion that is creating a valid m-adjustment sets. The only part of the algorithm that is responsible for aborting the recursion is line 7. ListSepCondition starts with a small set in each recursion path and in each run adds a variable to the set \tb{I}, if any of  independencies $(\tb{Y} \independent \tb{R}_{\tb{I}} \mid \tb{X})_{G_{\overline{X}}}\mbox{ and} \mbox{ }( (\tb{X} \cap An(\tb{R}_{\tb{I}})) \independent \tb{Y})_{G_{\underline{X}}}$ in line 7 do not hold at any step of recursion, it means by adding more variables to \tb{I} the dependency status wont change. Also, if FindSep cannot find an adjustment set for a given \tb{I} and \tb{E}, then there is not any set having \tb{I} as a subset of it that is adjustable. If a m-adjustment set blocks all non-causal paths, which means being separator, the FindSep should not return null for it. Therefore, the algorithm returns all sets \tb{Z} satisfying the conditions and is correct. \\
\\ \tb{Theorem }\tb{\ref{thmCorAlgListAdj}}\tb{(Correctness of ListMAdj)}. Given a m-graph $G$ and sets of disjoint variables \tb{X}, \tb{Y}, ListMAdj returns all the sets that satisfy the m-adjustment criterion relative to (\tb{X},\tb{Y}).
\\\tb{Proof}: ListMAdj function in the first line excludes all variables violating condition (a) in m-adjustment criterion in Def.~\ref{defMcriterion} and then calls ListSepCondition. Based on the theorem \ref{thmCorAlgListCond}, It is proved that ListSepCondition returns all candidates for m-adjustment sets satisfying m-adjustment criterion (b,c,d) in Def.~\ref{defMcriterion}. Therefore, the function returns all sets that satisfying m-adjustment criterion. \\
\\ \tb{Proposition} \tb{\ref{thmCpxAlg}} \tb{(Complexity of ListSepConditions)}.
ListSepConditions for a given graph $G$ has a time complexity of $O(n(n+m))$ polynomial delay where $n$ and $m$ are the number of variables and edges in $G$ respectively.\\
\tb{Proof:} To show the algorithm has a polynomial delay time complexity, we first demonstrate that it has an exponential time complexity and then we show it returns the first output as well as any two consecutive output in polynomial delay time. This algorithm examines all subset of variables in \tb{V} as candidate m-adjustment sets. The number of subsets is exponential to the size of \tb{V}. Therefore, algorithm has an exponential time complexity. Consider the recursion tree of ListSepConditions function.  For each node in this tree, the function checks the two independencies mentioned in line 7 and then calls FindSep function. If all conditions in line 7 satisfy, the algorithm goes to the next node in recursion. Checking the first two conditions requires $O(n+m)$ and FindSep has a time complexity of $O(n + m)$. Therefore the time needed for examining each node is $O(3(n+m)) = O((n+m))$. In order to print an output the recursion needs to reach to the leaf of the tree. Recursion at each step removes a variable from the potential variables that are in the set. The depth of the tree is equal to size of $n = \mid\tb{V}\mid$. Therefore, the time needed to reach to the end of the recursion and return the output is $O(n(n+m))$. For generating the next output, the algorithm needs to goes back from a leaf node to the next leaf node. In the worst case, consider all branched were aborted due to rejection of any of conditions in line 7. In this case, the algorithm needs to check all the nodes from the end of the tree to top of it. This is equal to depth of the tree. A tree can have a depth with a length of at most $n$ nodes. Therefore, generating the next output takes at most $O(n(n+m))$ time. \\
\\\tb{Theorem} \tb{\ref{thmListAdj}} \tb{(Complexity of ListMAdj)}.
	ListMAdj for a given graph $G$ returns all the m-adjustment sets with $O(n(n+m))$ polynomial delay where $n$ and $m$ are the number of variables and edges in $G$ respectively.\\     
\tb{Proof}: This function in the first part computes $D_{pcp}(\tb{X},\tb{Y})$. This can be done in polynomial time. Later, the function calls ListSepCondition function which has exponential time complexity with $O(n(n+m))$ polynomial delay. Therefore, the time complexity of entire algorithm is $O(n(n+m))$ polynomial delay. \\  
\\\tb{Theorem} \tb{ \ref{thmCorAlgListMin}} \tb{(Correctness of FindMinAdjSet)}.
Given a m-graph $G$ and disjoint sets of variables \tb{X}, \tb{Y}, FindMinAdjSet returns a m-adjustment set relative to (\tb{X},\tb{Y}) with minimum number of variables.\\
\tb{Proof}:
To prove FindMinAdjSet works properly, we prove this algorithm returns a valid m-adjustment and this m-adjustment has a minimum size. 
The algorithm excludes all variables that violate condition (a,c,d) in m-adjustment criterion in Def.~\ref{defMcriterion}. This function then find a minimum set \tb{D} in a proper back-door graph $G^{pbd}_{X,Y}$ by using a FindMinCostSep. Since a separator in a $G^{pbd}_{X,Y}$, blocks all non-causal path, the returned set will satisfies the m-adjustment condition (b). Note that it might be thought that adding $\tb{R}_{\tb{D}}$ variables to \tb{D} will open a blocked path. However, no $\tb{R}_{\tb{D}}$ lies on causal and non-causal path to \tb{Y} because of independecy condition between \tb{Y} and $\tb{R}_{\tb{D}}$. Therefore, this situation does not happen. Now we prove FindMinAdjSet returns the minimum size m-adjustment. It is proved that FinMinCostSep \cite{vanderZander:2014:CSA:3020751.3020845} returns a minimum separator in a graph $G$. Finding the minimum weight m-adjustment set in m-graph and causal graph is similar. The only difference between them is that in m-graph we have \tb{R} variables. We explained that no $\tb{R}_{\tb{D}}$ lies  on the path to \tb{Y}. Therefore, a set with \tb{R} variables as m-adjustment set cannot have the minimum size and it won't be returned by FindMinCostSep function.\\
\\\tb{Theorem} \tb{ \ref{propTimeComListMin}}
\tb{(Time Complexity FindMinAjdSet)}.
FindMinAdjSet has a time complexity of O($n^3$).\\
\tb{Proof}:
Consider a given graph $G=(\tb{V},\tb{E})$ with $|V|=n$ and $|E|= m$. Generating proper back-door graph from $G$ can be done in $O(m+n)$. The time complexity of computing $D_{pcp}$ is $O(n(n+m))$ also. Therefore, Line 3 has a complexity of $O(n(n+m))$. Testing the d-separation can be done in $O(n+m)$ and since we are checking d-separation for all nodes the line 4 and 5 has $On(n+m)$ complexity. The time complexity of FindMinCostSep is $O(n^3)$. Therefore, the time complexity of FindMinAjdSet is $O(n^3)$.

\section*{Appendix B: Proofs for Theorems 1 and 6}
\setlength\parindent{0pt}
 In the following section we used some of the proofs in \cite{correa2017causal}.

\begin{lem}
 Let \tb{X}, \tb{Y}, \tb{Z} be three disjoint sets of variables in an m-graph $G$ augmented with selection bias. If a set \tb{Z} satisfies the conditions in Def.~\ref{defMScriterion} for a given set of treatment and outcome \{\tb{X}, \tb{Y}\}, \tb{Z} can be partitioned into the sets bellow:
\end{lem}
	
\begin{enumerate}[\textbullet]
	\item
	$\tb{Z}^{Y,1}_{nd} =  \{ Z \mid Z \in \tb{Z} \setminus De_{\tb{X}}$ and $ (Z \independent \tb{Y} \mid \tb{X}, \tb{R}_{\tb{W}}, S)_{G_{\overline{X} }} \}$
	\item
	$\tb{Z}^{X,1}_{nd} =  \{ Z \mid  Z \in \tb{Z} \setminus De_{\tb{X}} \setminus \tb{Z}^{Y,1}_{nd}$ and $  ( Z \independent \tb{X} \mid \tb{Z}^{Y,1}_{nd}, \tb{R}_{\tb{W}}, S)_{G_{\overline{\tb{X}(\tb{R}_{\tb{W}}, S)} }} \}$
	
	\item
	
	$\tb{Z}^{Y}_{d} =  \{ Z \mid Z \in \tb{Z} \cap \ De_{\tb{X}}$ and $  ( Z \independent \tb{Y} \mid \tb{X}, \tb{Z}^{Y,1}_{nd} , \tb{Z}^{X,1}_{nd} ,\tb{R}_{\tb{W}}, S  )_{G_{\overline{X} }} \}$
	
	\item 
	$\tb{Z}^{X}_{d} =  \{ \tb{Z} \cap  De_{\tb{X}} \setminus \tb{Z}^{Y}_{d} \}$

	\item
	$\tb{Z}^{Y,2}_{nd} =  \{ Z \mid Z \in \tb{Z} \setminus \ De_{\tb{x}} \setminus \tb{Z}^{Y,1}_{nd} \setminus \tb{Z}^{X,1}_{nd}$ and $  ( Z \independent \tb{Y} \mid \tb{X} ,\tb{Z}^{Y,1}_{nd} , \tb{Z}^{X,1}_{nd} ,\tb{Z}^{Y}_{d} , \tb{Z}^{X}_{d}, \tb{R}_{\tb{W}}, S)_{G_{\overline{X} }} \}$	
	\item
	$\tb{Z}^{X,2}_{nd} = \tb{Z} \setminus De_{\tb{X}} \setminus \tb{Z}^{Y,1}_{nd} \setminus \tb{Z}^{X,1}_{nd}\setminus \tb{Z}^{Y,2}_{nd}$	
\end{enumerate}

Based on this partitioning, the following independencies can be conclude: \\$(\tb{Z}^{X}_{d} \independent \tb{X} \mid \tb{Z}^{Y}_{d},\tb{Z}^{Y,1}_{nd} ,\tb{Z}^{X,1}_{nd} ,\tb{R}_{\tb{W}}, S )_{G_{\overline{X(\tb{Z}^{Y}_{d} ,\tb{R}_{\tb{W}}, S)  }} } $ and $( \tb{Z}^{X,2}_{nd} \independent \tb{X} \mid \tb{Z} \setminus \tb{Z}^{X,2}_{nd},\tb{R}_{\tb{W}}, S)_{G_{\overline{X(\tb{Z}^{Y}_{d},\tb{Z}^{X}_{d} ,\tb{R}_{\tb{W}},S) } }}$.\\
\\\tb{Proof, Part 1}. To prove this independency holds: $( \tb{Z}^{X,2}_{nd} \independent \tb{X} \mid \tb{Z} \setminus \tb{Z}^{X,2}_{nd}, \tb{R}_{\tb{W}},S)_{G_{\overline{X(\tb{Z}^{Y}_{d},\tb{Z}^{X}_{d} ,\tb{R}_{\tb{W}},S)} }}$, we assume, by contradiction, that this assumption is not true. Therefore, there should be an open path between $Z'\in \tb{Z} \setminus \ De_{\tb{X}} \setminus\tb{Z}^{Y,1}_{nd}\setminus \tb{Z}^{X,1}_{nd} \setminus \tb{Z}^{Y,2}_{nd}$ and $ X \in \tb{X}$. We name this path $q$ in the graph $G_{\overline{X}}$. Since  $Z\textprime$ does not belong to $\tb{Z}^{Y,1}_{nd}$ and based on the definition of $\tb{Z}^{Y,1}_{nd}$, there exists an open path between $Y \in \tb{Y}$ and $Z\textprime$. We call this path $p$. The only collider that is allowed to exist in path $p$ is $Z\textprime$. $P$ cannot have any variable as a collider in $\{\tb{R}_{\tb{W}},S \}$ due to condition (c) that requires the d-separation between $Y$ and $\{\tb{R}_{\tb{W}},S \}$ for a given $X$. The variable $Z\textprime$ is not in $\tb{Z}^{Y,2}_{nd}$ based on the definition of $\tb{Z}^{X,2}_{nd}$. Therefore, $p$ does not contain any covariate in $\tb{Z}^{Y,1}_{nd},\tb{Z}^{X,1}_{nd},\tb{Z}^{Y}_{d},\tb{Z}^{X}_{d}$ or $\tb{Z}^{Y,2}_{nd}$; Otherwise, these sets close $p$ and lead $Z\textprime$ belongs to $\tb{Z}^{Y,2}_{nd}$ as per the fact that $p$ does not have any colliders. We have two situations: $X$ is or is not the ancestor of $\{\tb{R}_{\tb{W}},S\}$. In the first scenario, the arrow in path $q$ needs to come out of $X$. The definition of $Z\textprime$ necessitates that $Z\textprime$ is not a descendant of $X$. Therefore, there will be colliders in $q$. Due to the assumption that $q$ is open, these colliders must be ancestors of $\{\tb{R}_{\tb{W}},S\}$. This is in contradiction with the assumption that $X$ is not an ancestor of the variables in $\{\tb{R}_{\tb{W}},S\}$. \\
For the case that arrows coming into $X$, consider the joint path $p$ and $q$. In this path, $X$ should be an ancestor of $\{\tb{R}_{\tb{W}},S\}$ which is in contradiction to condition (d). If $Z\textprime$ is a collider in the joint path, we will have a non-causal open path which is against condition (b). If the arrows come out of $X$ in path $q$, due to the fact that  $Z\textprime$ is non-descendant of $X$, we need to have a collider in $q$. Based on our assumption $q$ is open. Therefore, the collider belongs to $\{\tb{R}_{\tb{W}},S\}$, otherwise $Z\textprime$ would be in the set $\tb{Z}^{X,1}_{nd}$. Having $ W \in \{\tb{R}_{\tb{W}},S\}$ as a collider in q necessitate $Z\textprime$ to be a collider by itself since we need to close the open path from $W$ to $Y$ based on condition (c). However, then conditioning on $Z\textprime$ will open the non-causal path from $\tb{X}$ to $\tb{Y}$. Therefore, the assumption of the existence of such $Z\textprime$ is invalid. \\ 
\\\tb{Proof, Part 2}. To prove this independency, 
$(\tb{Z}^{X}_{d} \independent \tb{X} \mid\tb{Z}^{Y}_{d},\tb{Z}^{Y,1}_{nd},\tb{Z}^{X,1}_{nd} , \tb{R}_{\tb{W}},S )_{G_{\overline{\tb{X}(\tb{Z}^{Y}_{d} ,\tb{R}_{\tb{W}},S)}} }$, we consider two cases of $\tb{Z}^{X}_{d} \ne 0$ and $\tb{Z}^{X}_{d} = 0$. For the first case, by contradiction, assume the independency is not true. Therefore, there exists an open path $q$ between $X$ and $Z\textprime \in \tb{Z}^{X}_{d}$, while the rest of $\tb{Z}^{Y}_{d},\tb{Z}^{Y,1}_{nd},\tb{Z}^{X,1}_{nd} , \tb{R}_{\tb{W}},S$ are observed. Since $Z\textprime$ belongs to $\tb{Z}^{X}_{d}$, all variables in path $q$ must be descendant of $X$. We know that $Z\textprime \notin \tb{Z}^{Y}_{d}$ based on the definition $\tb{Z}^{X}_{d}$. Therefore, there exists an open path $p$ from $Z\textprime$ to $Y$ while the variables $\tb{X}, \tb{Z}^{Y,1}_{nd}, \tb{Z}^{X,1}_{nd},\tb{Z}^{Y}_{d},\tb{R}_{\tb{W}},S$ are observed. There is no variable belongs to $\tb{X}$ in path $p$ based on the condition (b).  

Consider the junction of paths $q$ and $p$. Path $p$ cannot be directed since this junction path will be a proper causal path with some nodes from $Z\textprime$ on it. This is against condition (a). Therefore, $p$ needs to have colliders on it. 
Based on condition (b), $Z'$ cannot be collider, unless the path $q$ be closed by $\tb{Z}^{X,1}_{nd}$ and $\tb{Z}^{X,1}_{nd}$ which, based on their definition, is not possible to have them on $q$. 
If  $Z\textprime$ is not collider in $p$, there needs to be another collider variable $W \in \{\tb{R}_{\tb{W}},S\} \cup\tb{Z}^{Y,1}_{nd}\cup \tb{Z}^{X,1}_{nd} \cup \tb{Z}^{Y}_{d}$. None of these three sets can be collider. $\{\tb{R}_{\tb{W}},S\}$ cannot be collider because of condition (c). $\tb{Z}^{Y}_{d}$ cannot be collider since it is independent of $\tb{Y}$. Lastly, the two sets $\tb{Z}^{Y,1}_{nd}$  or $\tb{Z}^{X,1}_{nd}$ are descendant of $\tb{X}$. Therefore, they cannot be used as $W$. Since there is no valid variable to be as collider in path $p$, our assumption of existence of path $q$ is not a legitimate assumption. 
\\
\\
\tb{Theorem} \tb{\ref{thm-MSadj-com}} \tb{(MS-Adjustment)}  A set \tb{Z} is a ms-adjustment set for recovering causal effect of \tb{X} on \tb{Y} by the ms-adjustment formula in Definition~\ref{eq-ms-adj} if and only if it satisfies the ms-adjustment criterion in Definition~\ref{defMScriterion}.\\ 
\\\tb{Proof (if)}:
Based on lemma 1, a valid adjustment set \tb{Z} can be partitioned into the $\{\tb{Z}^{Y,1}_{nd}, \tb{Z}^{X,1}_{nd}, \tb{Z}^{Y}_{d}, \tb{Z}^{X}_{d},\tb{Z}^{Y,2}_{nd}, \tb{Z}^{X,2}_{nd} \}$. Based on this fact, the casual effect of \tb{X} on \tb{Y} can be computed as follows:
\\

According to condition (c), $(\tb{Y} \independent \{\tb{R}_{\tb{W}},S\}\mid \tb{X})_{G_{\overline{X}}} $ and $\{\tb{R}_{\tb{W}},S\}$ can be inserted into the following expression:

\begin{flalign}
&P(\tb{y} \mid do(\tb{x})) = P(\tb{y} \mid do(\tb{x}) , \tb{R}_{\tb{W}} = 1,S = 1) &
\end{flalign}

$\tb{Z}^{Y,1}_{nd}$ is independent of $\tb{Y}$. Therefore, it can be added to the first factor. Introducing the second factor with summation over 	$\tb{Z}^{Y,1}_{nd}$  values is valid. 

\begin{flalign}
&P(\tb{y} \mid do(\tb{x}))= \sum_{\tb{Z}^{Y,1}_{nd}}{} P(\tb{y} \mid do(\tb{x}) ,\tb{Z}^{Y,1}_{nd}, \tb{R}_{\tb{W}} = 1,S=1) P(\tb{Z}^{Y,1}_{nd} \mid \tb{R}_{\tb{W}} = 1,S=1)&   
\end{flalign}

By conditioning on $\tb{Z}^{X,1}_{nd}$ in the first factor, we get the following expression:

\begin{flalign}
\begin{split}
&P(\tb{y} \mid do(\tb{x}))= \sum_{\tb{Z}^{Y,1}_{nd} ,\tb{Z}^{X,1}_{nd} }{} P(\tb{y} \mid do(\tb{x}) ,\tb{Z}^{Y,1}_{nd},\tb{Z}^{X,1}_{nd}, \tb{R}_{\tb{W}}= 1,S=1) P(\tb{Z}^{X,1}_{nd} \mid do(\tb{x}),\tb{Z}^{Y,1}_{nd}, \tb{R}_{\tb{W}} = 1,S=1)\\
&P(\tb{Z}^{Y,1}_{nd} \mid \tb{R}_{\tb{W}} = 1,S=1)
\end{split}
\end{flalign}

In the second factor we can remove do(\tb{x}) based on the fact that $(\tb{Z}^{X,1}_{nd} \independent \tb{X}‌\mid\tb{Z}^{Y,1}_{nd}, \tb{R}_{\tb{W}},S)_{G_{\overline{X(\tb{R}_{\tb{W}},S)}}} $ ( $G_{\overline{\tb{X}(\tb{R}_{\tb{W}},S)}} = G_{\overline{\tb{X}(\tb{Z}^{Y,1}_{nd}, \tb{R}_{\tb{W}},S)}}$ since $\tb{Z}^{Y,1}_{nd}$ is independent of \tb{X}), we can use rule 3 of the do-calculus and remove do(\tb{x}). Taking advantage of the chain rule, factors two and three can be joined.  
\begin{flalign}
&P(\tb{y} \mid do(\tb{x}))=\sum_{\tb{Z}^{Y,1}_{nd} ,\tb{Z}^{X,1}_{nd} }{} P(\tb{y} \mid do(\tb{x}) ,\tb{Z}^{Y,1}_{nd},\tb{Z}^{X,1}_{nd}, \tb{R}_{\tb{W}} = 1,S=1) P(\tb{Z}^{Y,1}_{nd} \tb{Z}^{X,1}_{nd}\mid \tb{R}_{\tb{W}} = 1,S=1)&
\end{flalign}

 $\tb{Z}^{Y}_{d}$ is independent of \tb{Y}, so we can insert it in the first factor. $\tb{Z}^{Y,1}_{nd}$ can be inserted in the second factor and summed out on all its possible values. 

\begin{flalign}
&P(\tb{y} \mid do(\tb{x}))=\sum_{\tb{Z}^{Y,1}_{nd} ,\tb{Z}^{X,1}_{nd},\tb{Z}^{Y}_{d}}{} 
P(\tb{y} \mid do(\tb{x})  ,\tb{Z}^{Y,1}_{nd},\tb{Z}^{X,1}_{nd},\tb{Z}^{Y}_{d}, \tb{R}_{\tb{W}} = 1,S=1)
P(\tb{Z}^{Y,1}_{nd},\tb{Z}^{X,1}_{nd},\tb{Z}^{Y}_{d}\mid \tb{R}_{\tb{W}} = 1,S=1)& 
\end{flalign}
Since $\tb{Z}^{X}_{d}$ is not independent of Y, conditioning on it leads to:

\begin{flalign}
\begin{split}
&P(\tb{y} \mid do(\tb{x}))=\sum_{\tb{Z}^{Y,1}_{nd} ,\tb{Z}^{X,1}_{nd},\tb{Z}^{Y}_{d},\tb{Z}^{X}_{d}}{} 
P(\tb{y} \mid do(\tb{x}) ,\tb{Z}^{Y,1}_{nd},\tb{Z}^{X,1}_{nd},\tb{Z}^{Y}_{d},\tb{Z}^{X}_{d}, \tb{R}_{\tb{W}} = 1,S=1) 
P(\tb{Z}^{X}_{d} \mid  do(\tb{x}),\tb{Z}^{X,1}_{nd} ,\tb{Z}^{X,1}_{nd},\tb{Z}^{Y}_{d},\tb{R}_{\tb{W}} = 1,S=1)\\
&\times P(\tb{Z}^{Y,1}_{nd},\tb{Z}^{X,1}_{nd},\tb{Z}^{Y}_{d}\mid \tb{R}_{\tb{W}} = 1,S=1) 
\end{split}
\end{flalign}

do(\tb{x}) in the second factor can be removed by using rule 3 of the do-calculus, since the following independency: $(\tb{Z}^{X}_{d} \independent \tb{X} \mid\tb{Z}^{Y}_{d},\tb{Z}^{Y,1}_{nd},\tb{Z}^{X,1}_{nd}, \tb{R}_{\tb{W}} = 1,S=1)_{G_{\overline{\tb{X}(\tb{Z}^{Y}_{d}, \tb{R}_{\tb{W}},S)}}} $ holds. Then applying the chain rule on the factors two and three leads to the following expression. 

\begin{flalign}
\begin{split}
&P(\tb{y} \mid do(\tb{x}))=\sum_{\tb{Z}^{Y,1}_{nd} ,\tb{Z}^{X,1}_{nd}, 
	\tb{Z}^{Y}_{d},\tb{Z}^{X}_{d}}{} 
P(\tb{y} \mid do(\tb{x}) ,\tb{Z}^{Y,1}_{nd},\tb{Z}^{X,1}_{nd},\tb{Z}^{Y}_{d},\tb{Z}^{X}_{d}, \tb{R}_{\tb{W}} = 1,S=1) 
P(\tb{Z}^{Y,1}_{nd},\tb{Z}^{X,1}_{nd},\tb{Z}^{Y}_{d} ,\tb{Z}^{X}_{d}\mid \tb{R}_{\tb{W}} = 1,S=1)  
\end{split}
\end{flalign}

 We use this independency:$ (\tb{Z}^{Y,2}_{nd} \independent \tb{Y} \mid \tb{X},\tb{Z}^{Y,1}_{nd},\tb{Z}^{X,1}_{nd},\tb{Z}^{Y}_{d},\tb{Z}^{X}_{d}, \tb{R}_{\tb{W}} = 1,S=1)_{G_{\overline{X}}}$, and insert $\tb{Z}^{Y,2}_{nd}$ into the first factor. In the next step we add $\tb{Z}^{Y,2}_{nd}$ to the second factor and put a summation of it. 

\begin{flalign}
\begin{split}
& P(\tb{y} \mid do(\tb{x}))=\sum_{\tb{Z}^{Y,1}_{nd} ,\tb{Z}^{X,1}_{nd}, 
	\tb{Z}^{Y}_{d},\tb{Z}^{X}_{d},\tb{Z}^{Y,2}_{nd} }{} 
P(\tb{y}\mid do(\tb{x}), \tb{Z}^{Y,1}_{nd},\tb{Z}^{X,1}_{nd},\tb{Z}^{Y}_{d}, \tb{Z}^{X}_{d}, \tb{Z}^{Y,2}_{d} , \tb{R}_{\tb{W}} = 1,S=1 )\\& \times
P(\tb{Z}^{Y,1}_{nd},\tb{Z}^{X,1}_{nd},\tb{Z}^{Y}_{d},\tb{Z}^{X}_{d},\tb{Z}^{Y,2}_{d} \mid \tb{R}_{\tb{W}} = 1,S=1)
\end{split}
\end{flalign}

We condition on $\tb{Z}^{X,2}_{nd}$ in the first factor:

\begin{equation}
\begin{split}
&P(\tb{y} \mid do(\tb{x}))=\sum_{Z}{} 
P(\tb{y} \mid do(\tb{x}), \tb{Z}, \tb{R}_{\tb{W}} = 1,S=1 )
P(\tb{Z}^{X,2}_{nd} \mid do(\tb{x}), \tb{Z}^{Y,1}_{nd},\tb{Z}^{X,1}_{nd},\tb{Z}^{Y}_{d},\tb{Z}^{X}_{d},\tb{Z}^{Y,2}_{d}, \tb{R}_{\tb{W}} = 1,S=1)  \\    
&\times P(\tb{Z}^{Y,1}_{nd},\tb{Z}^{X,1}_{nd},\tb{Z}^{Y}_{d},\tb{Z}^{X}_{d},\tb{Z}^{Y,2}_{d} \mid \tb{R}_{\tb{W}} = 1,S=1) 
\end{split}
\end{equation}

By getting help from rule 3 of do-calculus and using this independency $ (\tb{Z}^{X,2}_{nd} \independent \tb{X} \mid \tb{Z}\setminus \tb{Z}^{X,2}_{nd}, \tb{R}_{\tb{W}} = 1,S=1)_{G_{\overline{\tb{X}(\tb{Z}^{Y}_{d},\tb{Z}^{X}_{d}, \tb{R})}}}$ do(\tb{x}) is removed from the second factor:

\begin{flalign}
&P(\tb{y} \mid do(\tb{x}))=\sum_{\tb{Z}}{}
P(\tb{y} \mid do(\tb{x}),\tb{z},\tb{R}_{\tb{W}} = 1,S=1)
P( \tb{z} \mid \tb{R}_{\tb{W}} = 1,S=1 )&
\end{flalign}

Based on conditions (a , b) we have $ (\tb{Y} \independent \tb{X} \mid \tb{Z}, \tb{R}_{\tb{W}} = 1,S=1)_{G_{\underline{X}}}$
\begin{flalign}
&P(\tb{y} \mid do(\tb{x}))=\sum_{\tb{Z}}{}
P(\tb{y} \mid \tb{x}, \tb{z}, \tb{R}_{\tb{W}} = 1,S=1 )
P( \tb{z} \mid \tb{R}_{\tb{W}} = 1,S=1  )&
\end{flalign}

\tb{Proof (only if):} In this part, we prove that if any of the criterion in MS-adjustment criterion is not true, there will be a graph $G$ that for a given set of treatment and outcome (\tb{X},\tb{Y}), the causal effect of $P(\tb{y} \mid do(\tb{x}))$ is not recoverable. The condition (b) in MS-adjustment criterion is the extended version of the condition (b) in adjustment set. The only difference is $\tb{R} \cup S$ is observed rather than only $S$. Therefore, we prove  besides \tb{Z}, $\tb{R} \cup S$ are required to block all non-causal paths. By contradiction, assume this is not the case. Therefore, there should be a non-causal path $q$ between $X$ and $Y$, that is closed by observed \tb{Z}, and gets open when there is a condition on $\tb{R} \cup S$. In order to demonstrate that the graph $G$ with this non-causal path is non-recoverable, we consider two models $M_1$ and $M_2$ both compatible with the graph $G$. We assign $P_1$ as a probability distubtion corresponding to $M_1$ and $P_2$ for $M_2$. $M_1$ and $M_2$ agree on probability distribution under selection, and MNAR biases and are disagree on the causal effect of the set of treatment on the set of outcome.

\begin{equation}
P_{1}(\tb{v} \mid\tb{R}^{\tb{v}}= 1, S = 1 )= P_{2}(\tb{v} \mid\tb{R}^{\tb{v}}= 1, S =1)
\label{models-agree}
\end{equation}
\begin{equation}
P_{1}(\tb{y} \mid do(\tb{x}) ) \ne P_{2}(\tb{y} \mid do(\tb{x}))
\end{equation}

 We construct $M_1$ in a way to be compatible with the graph $G_{\overline{\tb{R}_W, S}}$, separating all $\tb{R}_W,S$ from their parents, $(\tb{V} \independent \tb{R}_W \cup S)_{M_1}$  ,and $M_2$ compatible with the graph $G$: 
\begin{equation}
P_{1}(\tb{v} \mid \tb{R}^{\tb{v}}= 1, S = 1 )= P_{1}(\tb{v} \mid\tb{R}^{\tb{v}} \setminus \{\tb{R}_{\tb{W}} = 1, S =1\})
\end{equation}

 The causal effect query needs to be recoverable for any parametrization of probability distributions $P_1$, $P_2$. We construct $P_2$ in a way that equation \ref{models-agree} holds. 

Without loss of generality, we are considering the path between $Y' \in \tb{Y}$ and $X' \in \tb{X}$ that condition (b) of ms-adjustment criterion does not satisfy in it. Therefore, our desired model will have all the variables in the rest of the graph d-separated from the variables in the path. We have:

\begin{align}
&P(\tb{y} \mid do(\tb{x\textprime})) = \sum_{\tb{z}}{} P(\tb{y} \mid \tb{x\textprime},\tb{z},\tb{R}_W=1, S=1)P(\tb{z} \mid \tb{R}_W=1 , S=1) \\&
= \prod_{\tb{Y}}\sum_{\tb{z}} P(\tb{y} \mid \tb{x\textprime},\tb{z} ,\tb{R}_W=1 , S=1)P(\tb{z} \mid \tb{R}_W=1 , S=1) \\&
= ( \prod_{\tb{Y} \setminus Y'}P(\tb{y})) \sum_{\tb{z}}{} P(\tb{y}\textprime \mid \tb{x\textprime},\tb{z}, \tb{R}_W=1 , S=1)P(\tb{z} \mid \tb{R}_W=1 , S=1) \\&
= \gamma \sum_{\tb{z}} P(\tb{y}\textprime | \tb{x\textprime},\tb{z} ,\tb{R}_W=1 , S=1)P(\tb{z }\mid \tb{R}_W=1 , S=1)
\end{align}

$\gamma$ in above expression indicates product of marginal distribution $\tb{Y} \setminus {Y'}$. \\

The open non-causal path between $X\textprime$ and $Y\textprime$ that is blocked by \tb{Z} but opened with $\tb{R\textprime} \subseteq \tb{R}_{\tb{W}}=1 , S=1$, needs \tb{R\textprime} to be colliders. Fig.~\ref{fig-condB1} shows a general case for when the set \tb{R\textprime} has size 1. By a small change we will get Fig.~\ref{fig-condB2} which shows the general case for when \tb{R\textprime} has a size greater than one.

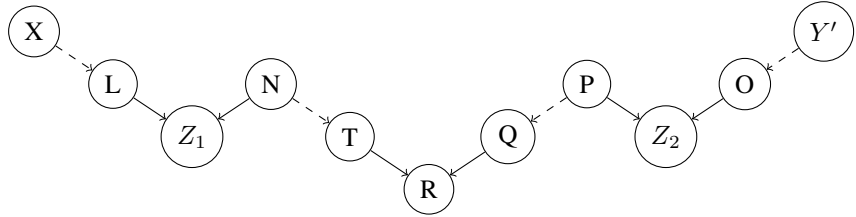
\begin{figure}[H] 
	\begin{center}
		\begin{tabular}{ccc}
		\begin{tikzpicture}[scale = 0.7]
		\node[circle,draw] (X) at (1,6) {X};
		\node[circle,draw] (L) at (2.5,5) {L};
		\node[circle,draw] (Z1) at (4,4) {$Z_1$};
		\node[circle,draw] (N) at (5.5,5) {N};
		\node[circle,draw] (T) at (7,4) {T};
		\node[circle,draw] (R) at (8.5,3) {R};
		\node[circle,draw] (Q) at (10,4) {Q};
		\node[circle,draw] (P) at (11.5,5) {P};
		\node[circle,draw] (Z2) at (13,4) {$Z_2$};
		\node[circle,draw] (O) at (14.5,5) {O};
		\node[circle,draw] (Y) at (16,6) {$Y'$};
		
		\draw[dashed, ->] (X)-- (L);
		\draw[->] (L)-- (Z1);
		\draw[->] (N)-- (Z1);
		\draw[dashed,->] (N)-- (T);
		\draw[->] (T)-- (R);
		\draw[->] (Q)-- (R);
		\draw[dashed, ->] (P)-- (Q);
		\draw[->] (P)-- (Z2);
		\draw[->] (O)-- (Z2);
		\draw[dashed, ->] (Y)-- (O);
		
		\end{tikzpicture}
		\end{tabular}
		\caption{This graph indicates an open non-causal path between $X$ and $Y$ with  conditioning on $\tb{R}$ and $\tb{Z}$. The path from $N$ to $T$ and from $P$ to $Q$ can be substituted by a path with any number of $Z \in \tb{Z}$. Dotted edges refer to chains of nodes. }
		\label{fig-condB1}
	\end{center}
\end{figure}  

	\tb{case 1:} There is only one collider belongs to \tb{R\textprime}. The proof for this part is as same as the selection bias \cite{correa2018generalized}. Therefore, we omit repeating it. 

\begin{figure}
	\begin{center}
		\begin{tabular}{ccc}
	\begin{tikzpicture}[scale = 0.8]
\node[circle,draw] (X) at (0,8) {X};
\node[circle,draw] (L0) at (1,6.5) {$L_0$};
\node[circle,draw] (Z1) at (2,5) {$Z_1$};
\node[circle,draw] (N) at (3,6) {$N_i$};
\node[circle,draw] (Li) at (4,5) {$L_i$};
\node[circle,draw] (Zi) at (5,4) {$Z'_i$};
\node[circle,draw] (Bi) at (6,5.5) {$B_i$};
\node[circle,draw] (Ti) at (7,4) {$T_i$};
\node[circle,draw] (Ri) at (8,2) {$R_i\mid S$};
\node[circle,draw] (Qi) at (9,4) {$Q_i$};
\node[circle,draw] (Pi) at (10,5.5) {$P_i$};
\node[circle,draw] (Z2i) at (11,4) {$Z''_i$};
\node[circle,draw] (Oi) at (12,5.5) {$O_i$};
\node[circle,draw] (Ni2) at (13,7) {$N_{i+1}$};
\node[circle,draw] (T) at (13.5,5.5) {T};
\node[circle,draw] (R) at (14,4) {$R\mid S$};
\node[circle,draw] (Q) at (15,5) {Q};
\node[circle,draw] (P) at (16,6) {P};
\node[circle,draw] (Z2) at (17,5) {$Z_2$};
\node[circle,draw] (O) at (18,6.5) {O};
\node[circle,draw] (Y) at (19,8) {$Y'$};

\draw[dashed, ->] (X)-- (L0);
\draw[->] (L0)-- (Z1);
\draw[->] (N)-- (Z1);
\draw[dashed,->] (N)-- (Li);
\draw[->] (Li)-- (Zi);
\draw[->] (Bi)-- (Zi);
\draw[dashed, ->] (Bi)-- (Ti);
\draw[->] (Ti)-- (Ri);
\draw[->] (Qi)-- (Ri);
\draw[dashed, ->] (Pi)-- (Qi);
\draw[->] (Pi)-- (Z2i);
\draw[->] (Oi)-- (Z2i);
\draw[dashed,->] (Ni2) -- (Oi);
\draw[dashed,->] (Ni2) -- (T);
\draw[->] (T)-- (R);
\draw[->] (Q)-- (R);
\draw[dashed, ->] (P)-- (Q);
\draw[->] (P)-- (Z2);
\draw[->] (O)-- (Z2);
\draw[dashed, ->] (Y)-- (O);

\end{tikzpicture}
		\end{tabular}
\caption{The path from $N_i$ to $N_{i+1}$ can recursively be substituted by more of the path of the same kind to include arbitrary number of $\tb{R}$ variables. $R_i\mid S$ and $R\mid S$ indicate that the variables can belong to either $\tb{R}_{\tb{W}}$ or $S$ by considering the fact that there is only one $S$ in the path.} 
\label{fig-condB2} 
\end{center}
\end{figure}
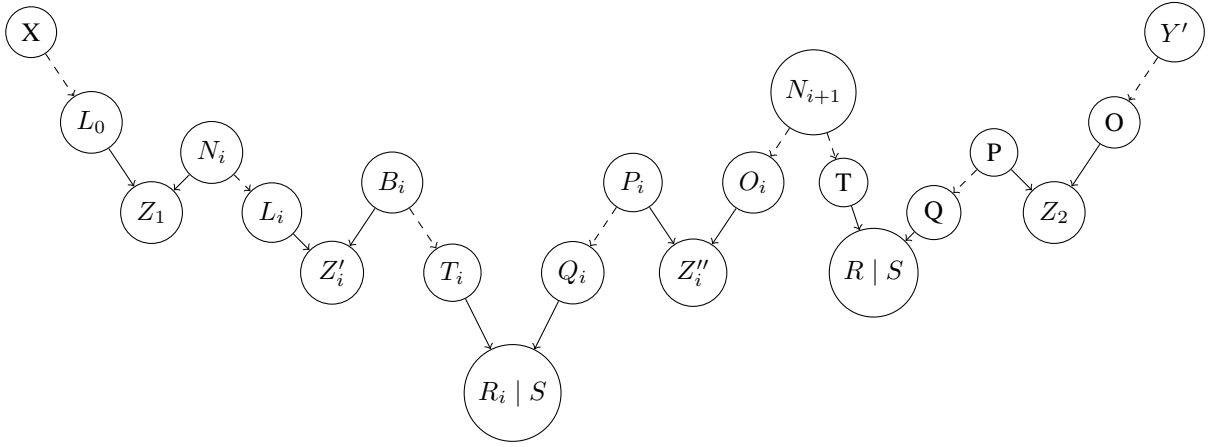  

\tb{case 2:}
The set \tb{R\textprime} might have the size greater than one. Fig.~\ref{fig-condB2} expresses a graphical representation for this situation. To prove this case, we provide a parametrization for the path from $N_i$ to $N_{i+1}$. The rest of the proof will be similar to case 1.

We assign $P_1(N_{i+1}) = P_1(P_i) = P_1(B_i) = P_1(N_i) = 1/2, P(O_i \mid N_{i+1}) = 1/2 +‌\epsilon_{5}/2, P(O_i \mid  \overline{N_{i+1}}) = 1/2 -‌\epsilon_{5}/2, P(L_i \mid N_{i}) = 1/2 +‌\epsilon_{6}/2, P(L_i \mid \overline{N_{i}}) = 1/2 -‌\epsilon_{6}/2, P(T_i \mid B_{i}) = 1/2 +‌\epsilon_{7}/2, P(T_i \mid \overline{B_{i}}) = 1/2 -‌\epsilon_{7}/2, P(Q_i \mid P_{i}) = 1/2 +‌\epsilon_{8}/2, P(Q_i \mid \overline{P_{i}}) = 1/2 -‌\epsilon_{8}/2 , P(z''_i \mid P_i, O_i) =P(z''_i \mid \overline{P_i}, O_i)=P(z''_i \mid \overline{O_i}, P_i)=P(z''_i \mid \overline{P_i},\overline{ O_i}) = 1/2 , P(R_iS \mid T_i, Q_i) =P(R_iS \mid \overline{T_i}, Q_i)=P(R_iS \mid \overline{T_i}, Q_i)=P(R_iS \mid \overline{T_i},\overline{Q_i}) = 1/2,  P(Z'_i \mid B_i, L_i) =P(Z'_i \mid \overline{B_i}, L_i)=P( Z'_i \mid \overline{B_i}, L_i)=P(Z_i \mid \overline{B_i},\overline{L_i}) = 1/2$, Where $\epsilon_{i} = (\frac{1}{5}^{k_{i}})$ , $k_{5}$ is the length of the path $N_{i+1}$ to $O_i$, $k_{6}$ is the length of the path from $N_i$ to $L_i$,$k_{7}$ is the length of the path from $B_i$ to $T_i$, and $k_{8}$ is  the length of the path from $P_i$ to $Q_i$. This parametrization provides the same values for $Q_1$ and $Q_2$ as case 1. \\

Now we evaluate the necessity of condition (c). Fig.~\ref{fig-31} and 
Fig.~\ref{fig-32} show all the cases violating condition (c). Note that in these  figures, $R_1, R_2,R_3, R_i \in \tb{R}_{\tb{W}}$, and, by mentioning $R_i\mid S$, we are referring to have either S or $R_i$ violated condition (c). The proof for cases 1 to 6 is similar as \cite{correa2018generalized}. Cases of 7, 10, 11, and 12 are extended versions of case 2, and case 8, 13, and 14 are extended versions of case 3. It is clear that by adding more edges to case 5 we can obtain case 9. We can conclude these extended versions are not recoverable, since if recoverability is impossible in a graph, adding more edges does not change recoverability status.

\begin{figure} 
	\begin{center}
		\begin{tabular}{cc}
				\begin{subfigure}[normal]{0.4\linewidth}
				\begin{tikzpicture}[scale = 0.5]
				\node[circle,draw] (X) at (10,6) {X};
				\node[circle,draw] (Y) at (12,6) {$Y'$};
				\node[circle,draw] (Z) at (13,4) {$Z$};
				\node[circle,draw] (T) at (14,2) {T};
				\node[circle,draw] (R) at (16,0) {$R_1 \mid S$};
				
				\draw[dashed, ->] (Y)-- (Z);
				\draw[dashed,->] (Z)--(T);
				\draw[->] (T)--(R);
				\end{tikzpicture}
				\caption{case 2} 
			\end{subfigure}
			 &  	\begin{subfigure}[normal]{0.4\linewidth}
			 	\begin{tikzpicture} [scale=0.5]
			 	\node[circle,draw] (X) at (10,2) {X};
			 	\node[circle,draw] (Y) at (13,4) {$Y'$};
			 	\node[circle,draw] (R) at (13,0) {$R_1 \mid S$};
			 	
			 	\draw[dashed, ->] (X)-- (Y);
			 	\draw[->] (Y)--(R);
			 	
			 	\end{tikzpicture}%
			 	\caption{case 1}  
			 \end{subfigure}\\
		 \begin{subfigure}[normal]{0.4\linewidth}
		 	\begin{tikzpicture}[scale = 0.5]
		 	\node[circle,draw] (X) at (10,6) {X};
		 	\node[circle,draw] (Y) at (12,6) {$Y'$};
		 	\node[circle,draw] (Z) at (13,4) {$Z$};
		 	\node[circle,draw] (T) at (14,2) {T};
		 	\node[circle,draw] (R) at (16,0) {$R_1 \mid S$};
		 	
		 	\draw[dashed, ->] (Y)-- (Z);
		 	\draw[dashed,->] (Z)--(T);
		 	\draw[->] (T)--(R);
		 	\end{tikzpicture}
		 	\caption{case 3}  
		 \end{subfigure}  &
		\begin{subfigure}[normal]{0.4\linewidth}
		\begin{tikzpicture}[scale = 0.5]
		\node[circle,draw] (X) at (2,2) {X};
		\node[circle,draw] (Y) at (5,2) {$Y'$};
		\node[circle,draw] (Q) at (6.5,3) {Q};
		\node[circle,draw] (N) at (8,5) {N};
		\node[circle,draw] (T) at (8,3) {T};
		\node[circle,draw] (R) at (8,0) {$R_1 \mid S$};
		
		\draw[dashed, ->] (X)-- (Y);
		\draw[dashed, ->] (N)-- (Q);
		\draw[dashed, ->] (N)-- (T);
		\draw[->] (Q)--(Y);
		\draw[ ->] (T)-- (R);
		\end{tikzpicture}
		\caption{case 4}  
	\end{subfigure}   \\
		\begin{subfigure}[normal]{0.4\textwidth}
		\begin{tikzpicture}[scale = 0.5]
		\node[circle,draw] (X) at (2,2) {X};
		\node[circle,draw] (Y) at (5,2) {$Y'$};
		\node[circle,draw] (Q) at (6.5,3) {Q};
		\node[circle,draw] (Z) at (8,5) {$Z$};
		\node[circle,draw] (T) at (8,3) {T};
		\node[circle,draw] (R) at (8,0) {$R_1 \mid S$};
		
		\draw[dashed, ->] (X)-- (Y);
		\draw[dashed, ->] (Z)-- (Q);
		\draw[dashed, ->] (Z)-- (T);
		\draw[->] (Q)--(Y);
		\draw[ ->] (T)-- (R);
		\end{tikzpicture}
		\caption{case 5} 
		\end{subfigure}
		&
		\begin{subfigure}[normal]{0.4\textwidth}
		\begin{tikzpicture}[scale = 0.5]
		\node[circle,draw] (X) at (2,2) {X};
		\node[circle,draw] (Y) at (5,2) {$Y'$};
		\node[circle,draw] (R) at (7,0) {$R_1 \mid S$};
		
		\draw[dashed, ->] (X)-- (Y);
		\draw[dashed, <->] (Y)to[out=+20,in=+70] (R);
		\end{tikzpicture}
		\caption{case 6} \label{fig:M6}  
		\end{subfigure} \\
		\begin{subfigure}[normal]{0.3\textwidth}
		\begin{tikzpicture}[scale = 0.5]
		\node[circle,draw] (X) at (0,2) {X};
		\node[circle,draw] (Y) at (2,2) {$Y'$};
		\node[circle,draw] (T) at (4,1) {T};
		\node[circle,draw] (R1) at (6,0) {$R_1$};
		\node[circle,draw] (R2) at (8,2) {$R_2$};
		
		\draw[->, dashed] (X)-- (Y);	
		\draw[->, dashed] (Y)-- (T);
		\draw[->] (T)-- (R1);
		\draw[thick,dashed, ->] (R2)-- (R1);
		\end{tikzpicture}
		\caption{case 7}  
		\end{subfigure}
	 	&
		\begin{subfigure}[normal]{0.4\textwidth}
		\begin{tikzpicture}[scale = 0.5]
		\node[circle,draw] (X) at (0,6) {X};
		\node[circle,draw] (Y) at (2,6) {$Y'$};
		\node[circle,draw] (Z) at (2.5,4) {$Z_i$};
		\node[circle,draw] (T) at (4,3) {$T$};
		\node[circle,draw] (R1) at (6,3) {$R_1$};
		\node[circle,draw] (R2) at (8,4) {$R_2$};
		
		\draw[->, dashed] (Y)-- (Z);
		\draw[->, dashed] (Z)-- (T);
		\draw[->] (T)-- (R1);
		\draw[thick,dashed,->] (R2)-- (R1);
		\end{tikzpicture}
		\caption{case 8}  
		\end{subfigure} \\
	\begin{subfigure}[normal]{0.4\textwidth}
		\begin{tikzpicture}[scale = 0.5]
		\node[circle,draw] (X) at (0,2) {X};
		\node[circle,draw] (Y) at (2,2) {$Y'$};
		\node[circle,draw] (L) at (4,2) {L};
		
		\node[circle,draw] (Z) at (5,4) {Z};
		\node[circle,draw] (Q) at (6.5,2) {Q};	
		\node[circle,draw] (R1) at (8,0) {$R_1$};
		\node[circle,draw] (R2) at (10,2) {$R_2$};
		
		\draw[->, dashed] (X)-- (Y);	
		\draw[->, dashed] (Z)-- (L);
		\draw[->] (L)-- (Y);
		\draw[->, dashed] (Z)-- (Q);
		\draw[->] (Q)-- (R1);
		\draw[thick,dashed, ->] (R2)-- (R1);
		\end{tikzpicture}
		\caption{case 9}  
	\end{subfigure} & 
\begin{subfigure}[normal]{0.4\textwidth}
	\begin{tikzpicture}[scale = 0.5]
	\centering
	\node[circle,draw] (X) at (0,4) {X};
	\node[circle,draw] (Y) at (2,4) {$Y'$};
	\node[circle,draw] (Q) at (3.5,2) {Q};
	\node[circle,draw] (R1) at (5,0) {$R_1 \mid S$};
	\node[circle,draw] (L) at (6,2) {L};	
	\node[circle,draw] (T) at (7,4) {T};
	\node[circle,draw] (P) at (8,2) {P};	
	\node[circle,draw] (R2) at (9,0) {$R_2$};
	\node[circle,draw] (R3) at (11,2) {$R_3$};
	
	\draw[->, dashed] (X)-- (Y);
	\draw[->, dashed] (Y)-- (Q);
	\draw[->] (Q)-- (R1);
	\draw[->] (L)-- (R1);
	\draw[dashed,->] (T)-- (L);
	\draw[dashed,->] (T)-- (P);
	\draw[->] (P)-- (R2);
	\draw[thick,dashed,->] (R3)-- (R2);
	
	\end{tikzpicture}
	\caption{case 10} 
\end{subfigure}\\

\end{tabular}
\caption{ All cases condition (c) are violated.  \label{fig-31}}
\end{center}
\end{figure}
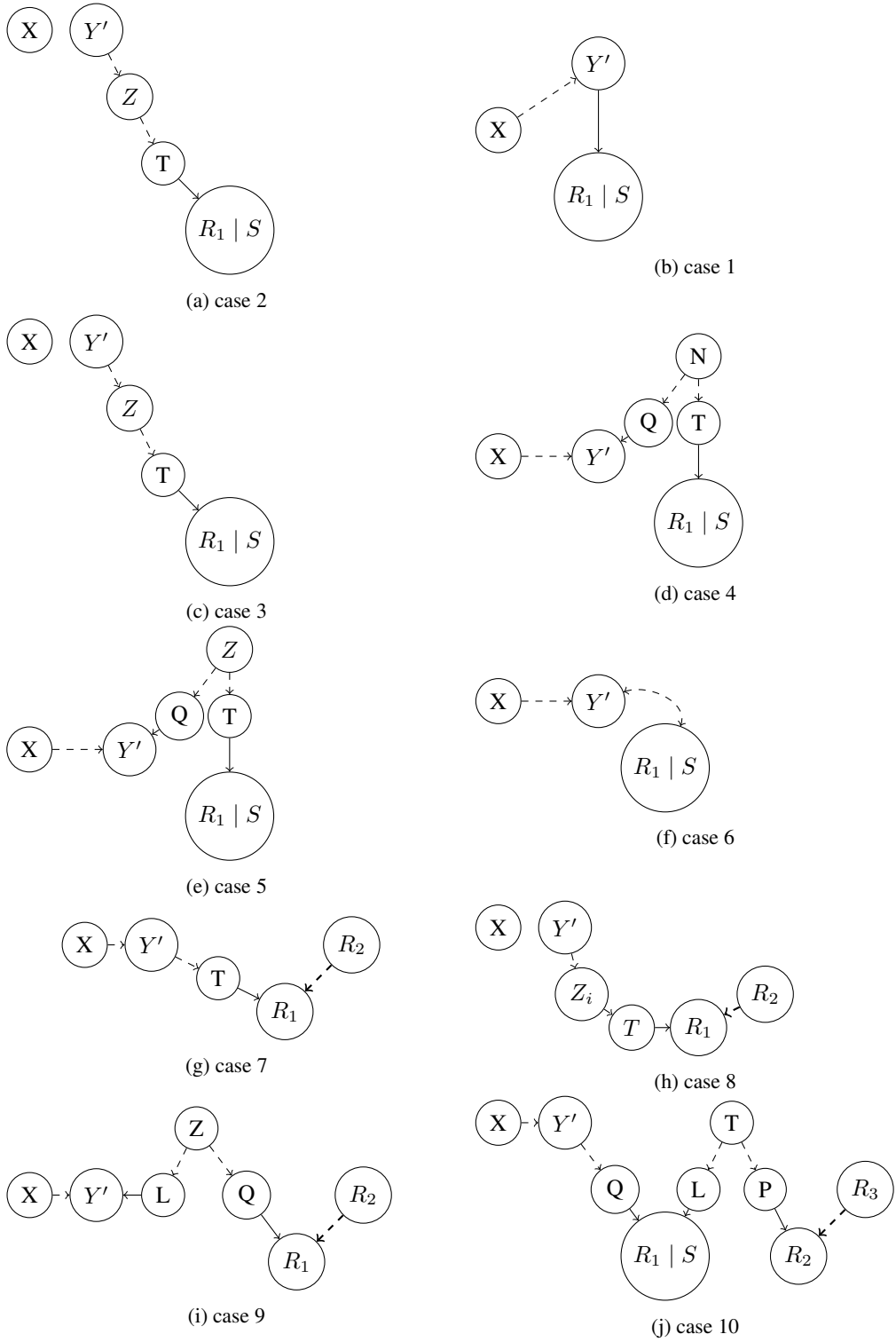

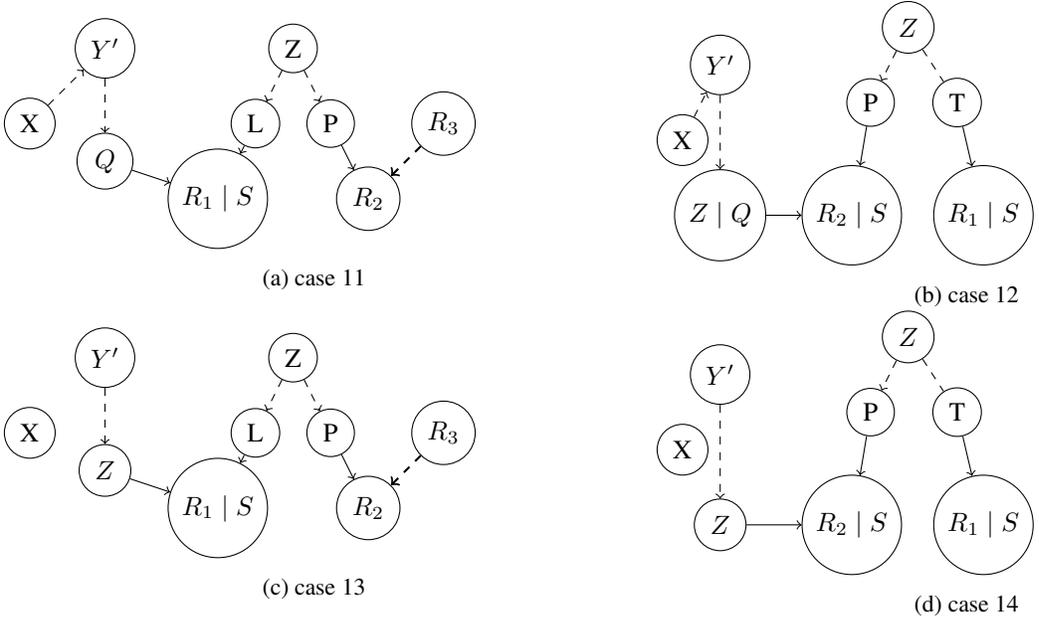
\begin{figure} 
\begin{center}
\begin{tabular}{cc}

\begin{subfigure}[normal]{0.5\textwidth}
	\begin{tikzpicture}[scale = 0.5]
	\node[circle,draw] (X) at (0,2) {X};
	\node[circle,draw] (Y) at (2,4) {$Y'$};
	\node[circle,draw] (Q) at (2,1) {$Q$};
	\node[circle,draw] (R1) at (5,0) {$R_1 \mid S$};
	\node[circle,draw] (L) at (6,2) {L};	
	\node[circle,draw] (Z) at (7,4) {Z};
	\node[circle,draw] (P) at (8,2) {P};
	\node[circle,draw] (R2) at (9,0) {$R_2$};
	\node[circle,draw] (R3) at (11,2) {$R_3$};
	
	\draw[->, dashed] (X)-- (Y);	
	\draw[dashed, ->] (Y)-- (Q);
	\draw[ ->] (Q)-- (R1);
	\draw[dashed,->] (Z)-- (L);
	\draw[->] (L)-- (R1);
	\draw[dashed,->] (Z)-- (P);
	\draw[->] (P)-- (R2);
	\draw[thick, dashed,->] (R3)-- (R2);
	
	\end{tikzpicture}
	\caption{case 11} 
\end{subfigure}
&
\begin{subfigure}[normal]{0.5\textwidth}
	\begin{tikzpicture}[scale = 0.5]
	\centering
	\node[circle,draw] (X) at (0,2) {X};
	\node[circle,draw] (Y) at (1,4) {$Y'$};
	\node[circle,draw] (Q) at (1,0) {$Z \mid Q$};
	\node[circle,draw] (R2) at (4.5,0) {$R_2 \mid S$};
	\node[circle,draw] (P) at (5,3) {P};
	\node[circle,draw] (Z) at (6,5) {$Z$};
	\node[circle,draw] (T) at (7.3,3) {T};
	\node[circle,draw] (R1) at (8,0) {$R_1 \mid S$};
	
	\draw[dashed, ->] (X)-- (Y);
	\draw[dashed, -] (Z)-- (T);
	\draw[dashed, ->] (Z)-- (P);
	\draw[->] (P)-- (R2);
	\draw[dashed, ->] (Y) --(Q);
	\draw[->] (Q) --(R2);
	\draw[ ->] (T)-- (R1);
	\end{tikzpicture}
	\caption{case 12}   
\end{subfigure}\\

\begin{subfigure}[normal]{0.5\textwidth}
	\begin{tikzpicture}[scale = 0.5]
	\node[circle,draw] (X) at (0,2) {X};
	\node[circle,draw] (Y) at (2,4) {$Y'$};
	\node[circle,draw] (Q) at (2,1) {$Z $};
	\node[circle,draw] (R1) at (5,0) {$R_1 \mid S$};
	\node[circle,draw] (L) at (6,2) {L};	
	\node[circle,draw] (Z) at (7,4) {Z};
	\node[circle,draw] (P) at (8,2) {P};
	\node[circle,draw] (R2) at (9,0) {$R_2$};
	\node[circle,draw] (R3) at (11,2) {$R_3$};
	\
	
	\draw[dashed, ->] (Y)-- (Q);
	\draw[ ->] (Q)-- (R1);
	\draw[dashed,->] (Z)-- (L);
	\draw[->] (L)-- (R1);
	\draw[dashed,->] (Z)-- (P);
	\draw[->] (P)-- (R2);
	\draw[thick, dashed,->] (R3)-- (R2);
	\end{tikzpicture}
	\caption{case 13}
\end{subfigure}
&
\begin{subfigure}[normal]{0.5\textwidth}
	\begin{tikzpicture}[scale = 0.5]
	\node[circle,draw] (X) at (0,2) {X};
	\node[circle,draw] (Y) at (1,4) {$Y'$};
	\node[circle,draw] (Q) at (1,0) {$Z$};
	\node[circle,draw] (R2) at (4.5,0) {$R_2 \mid S$};
	\node[circle,draw] (P) at (5,3) {P};
	\node[circle,draw] (Z) at (6,5) {$Z$};
	\node[circle,draw] (T) at (7.3,3) {T};
	\node[circle,draw] (R1) at (8,0) {$R_1 \mid S$};

	\draw[dashed, -] (Z)-- (T);
	\draw[dashed, ->] (Z)-- (P);
	\draw[->] (P)-- (R2);
	\draw[dashed, ->] (Y) --(Q);
	\draw[->] (Q) --(R2);
	\draw[ ->] (T)-- (R1);
	\end{tikzpicture}
	\caption{case 14}
\end{subfigure}

		\end{tabular}
		\caption{All cases condition (c) violated.  \label{fig-32}}
	\end{center}
\end{figure}

In this part we are evaluating whether the condition (d) is necessary or not. To condition (d) be violated, there should be a back-door path between $X \in \tb{X}\cap An(\tb{R}_{\tb{W}})$ and $Y\textprime \in \tb{Y}$. We name this path $p$. Based on condition (b), this path should be blocked by some $Z\textprime \in \tb{Z}$. There are two scenarios. The first one is, we have a direct causal path from $X$ to $Y\textprime$ and the second one is, lack of existence of that type of path, both cases are demonstrated in Fig.~\ref{fig-Cond}. The proof for these cases are similar to selection bias one and is provided in \cite{correa2018generalized}.

\begin{figure} 
	\begin{center}
		\begin{tabular}{ccc}
	\begin{subfigure}[normal]{0.5\textwidth}
	\centering
	\begin{tikzpicture}[scale = 0.5]
	\node[circle,draw] (X) at (4,4) {X};
	\node[circle,draw] (Y) at (8,4) {Y\textprime};
	\node[circle,draw] (Z0) at (6,6) {$Z\textprime$};
	\node[circle,draw] (R) at (0,4) {$R \mid S$};
	
	\draw[dashed, ->] (Z0)-- (X);
	\draw[dashed, ->] (Y)-- (Z0);
	\draw[dashed, ->] (X)-- (R);
	\end{tikzpicture}
	\caption{case 1} \label{fig:M2}  
\end{subfigure}
 &
 
	\begin{subfigure}[normal]{0.5\textwidth}
	\centering	
	\begin{tikzpicture}[scale = 0.5]
	\node[circle,draw] (X) at (4,4) {X};
	\node[circle,draw] (Y) at (8,4) {Y\textprime};
	\node[circle,draw] (Z) at (6,6) {Z\textprime};
	\node[circle,draw] (R) at (0,4) {$R \mid S$};
	
	\draw[dashed,->] (X) -- (Y);
	\draw[dashed,->] (Z) -- (X);
	\draw[dashed,->] (Z) -- (Y);
	\draw[dashed,->] (X)-- (R);
	\end{tikzpicture}
	\caption{case 2} \label{fig:M1}  
\end{subfigure}\\

\end{tabular}
\caption{Cases considered for the necessity of condition (d).   \label{fig-Cond}}
	\end{center}
\end{figure}
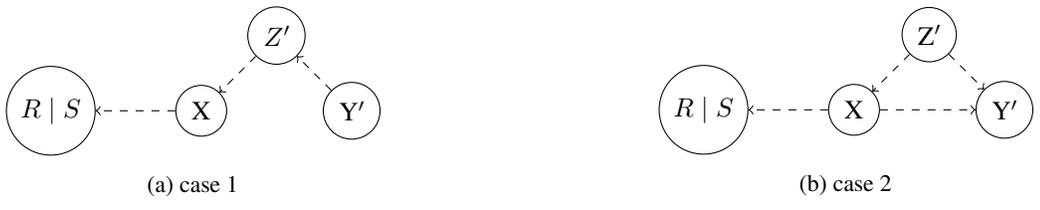

\noindent\\\tb{Theorem} \tb{\ref{thm-Madj-com}} [M-Adjustment]. A set \tb{Z} is a m-adjustment set for recovering causal effect of \tb{X} on \tb{Y} by the m-adjustment formula in Def.~\ref{defMadjustment} if and only if it satisfies the M-adjustment criterion in Def.~\ref{defMcriterion}.  \\
\tb{Proof}: The proof of this theorem is almost similar to the proof of theorem.~\ref{thm-MSadj-com} with the difference that here we have $\tb{R}_{\tb{W}}$ instead of $\tb{R}_{\tb{W}} \cup S$. \\

\end{document}